\documentclass{article}

\usepackage[preprint]{neurips_2026}

\usepackage[utf8]{inputenc}
\usepackage[T1]{fontenc}
\usepackage{amsmath,amssymb,amsfonts,amsthm}
\usepackage{algorithm}
\usepackage{algorithmic}
\usepackage{booktabs}
\usepackage{graphicx}
\usepackage[hidelinks]{hyperref}
\usepackage{xcolor}
\usepackage{multirow}
\usepackage{makecell}
\usepackage{subcaption}
\usepackage{float}
\usepackage{tikz}
\usetikzlibrary{positioning, arrows.meta, fit, calc, shapes.geometric}
\usepackage{microtype}
\usepackage{rotating}

\newcommand{\cedar}{\textsc{CEDAR}}
\newcommand{\cdnots}{\textsc{CDNOTS}}
\newcommand{\sypi}{\textsc{SyPI}}
\newcommand{\RR}{\mathbb{R}}
\newcommand{\EE}{\mathbb{E}}

\DeclareMathOperator{\dCov}{dCov}

\newtheorem{proposition}{Proposition}
\newtheorem{theorem}{Theorem}

\title{\cedar{}: Causal Edge Discovery for Autoregressive Processes}

\author{
  Mohammad Fesanghary \\
  Bloomberg LP \\
  \texttt{mfesanghary1@bloomberg.net}
}

\date{}

\begin{document}
\maketitle

\begin{abstract}
We propose \cedar{} (\textbf{C}ausal \textbf{E}dge \textbf{D}iscovery for \textbf{A}uto\textbf{r}egressive Processes), a constraint-based method for lagged causal edge discovery in sparse autoregressive time series.
\cedar{} screens candidate cross-variable lags using AR(1)-residualized, $\mathcal{U}$-centered distance correlation, then applies two targeted conditional-independence tests per significant cross-variable lag candidate and accepts at most one lag per ordered pair.
A stable MCI pruning step removes indirect edges, and optional deterministic C-nodes adjust for specified trend-like nonstationarity.
In sparse regimes where few lags survive screening, \cedar{} requires $O(d^2)$ CI tests after screening while retaining edge-level interpretability.
\cedar{} is most effective when data are scarce and variables exhibit lag-1 self-dynamics; methods with richer conditioning sets become preferable as $T$ grows or when higher-order autoregressive or simultaneous multi-lag effects are common.
\end{abstract}

\renewcommand{\thefootnote}{\fnsymbol{footnote}}
\footnotetext[1]{Source code and experiment scripts are available in the \textsc{Causal-TS} package; all results are fully reproducible.}
\renewcommand{\thefootnote}{\arabic{footnote}}

\section{Introduction}
\label{sec:intro}

Causal discovery from multivariate time series is central to fields such as climate science, finance, neuroscience, and industrial process monitoring~\citep{peters2017elements, runge2019detecting, spirtes2000causation}.
Time-series causal discovery is substantially more challenging than the i.i.d.\ setting: observations are temporally dependent, variables may influence one another across multiple lags, and strong autocorrelation or common drivers can create spurious associations.
Although temporal ordering provides useful information---causes must precede their effects---it also introduces a larger and more structured search space over lagged variables.
These challenges are further amplified in real-world systems, where causal mechanisms may be nonstationary and variables may depend on higher-order autoregressive structure.
Reliable causal discovery from time series therefore requires methods that can distinguish direct causal effects from autocorrelation, lagged dependence, indirect links, and changing mechanisms over time.

\paragraph{Related work.}
Causal discovery algorithms for time series broadly fall into two paradigms.
\emph{Score-based methods}---DYNOTEARS~\citep{pamfil2020dynotears}, VARLiNGAM~\citep{hyvarinen2010estimation}, Neural Granger Causality~\citep{tank2021neural}, CUTS+~\citep{cheng2023cuts_plus}---formulate discovery as an optimization problem, searching for the graph that best explains the data according to a scoring criterion.
They scale well with dimensionality and can naturally incorporate sparsity penalties, but require post-hoc threshold tuning for binary graph output and make individual edges harder to interpret since they arise from a global objective rather than explicit evidence for each link.
\emph{Constraint-based methods}---PCMCI/PCMCI+~\citep{runge2019detecting, runge2020discovering}, LPCMCI~\citep{gerhardus2020high}, \cdnots{}~\citep{sadeghi2025cdnots}---provide a more direct and interpretable statistical procedure, adding or removing edges based on conditional independence (CI) tests.
Their main cost is statistical and computational: conditioning sets can grow quickly in dense or high-dimensional graphs, reducing test power when samples are limited.

\paragraph{Our approach.}
To address this low-sample regime, we introduce \cedar{}, which constructs targeted conditioning sets from temporal structure rather than searching over large parent subsets.
The method retains edge-level interpretability: every reported link is supported by explicit CI tests and then rechecked by a stable MCI pruning step.
\cedar{} builds on \sypi{}~\citep{mastakouri2021sypi}, a causal feature-selection method for autoregressive time series with a designated sink target.
We extend the idea to all ordered pairs by automating lag selection, relaxing the sink-target restriction through temporal acyclicity, and pruning indirect links.
The resulting procedure is a practical finite-sample method for sparse AR(1)-like systems; its oracle guarantee requires additional structural conditions stated in Appendix~\ref{sec:theory}.

\paragraph{Contributions.}
Our main contributions are:
\begin{enumerate}
  \item \textbf{Nonlinear lag selection with analytic screening}: The original \sypi{} uses Lasso regression for lag selection, which is limited to linear dependencies. We replace it with $\mathcal{U}$-centered distance correlation~\citep{szekely2014partial}---a nonparametric measure that captures arbitrary nonlinear associations and has zero expectation under independence---paired with an analytic $t$-test~\citep{szekely2013dcor} for fast screening, enabling \cedar{} to screen candidate lags in both linear and nonlinear systems.
  \item \textbf{Stable MCI pruning}: After the initial two-condition CI testing, some detected edges may be indirect---mediated through other variables. We introduce a stable, order-independent Momentary Conditional Independence (MCI) pruning step that conditions on discovered parents of both source and target, removing all such edges in a single pass without the order dependence that affects iterative approaches.
  \item \textbf{Nonstationarity handling via C-node}: Real-world time series often exhibit distributional changes over time that create spurious associations between variables sharing a common trend. We introduce a deterministic C-node encoding polynomial time trends, treated as a pure cause, that adjusts for specified deterministic trend-like nonstationarity before it produces false edges---a capability absent from the original \sypi{} framework.
\end{enumerate}

\section{Background and Notation}
\label{sec:background}

Let $\mathbf{X}=(X^1,\ldots,X^d)$ be a $d$-variate time series with $T$ observations and maximum lag $L$.
We focus on the cross-variable lagged graph: $G_X[c,e,\ell]=1$, for $c\neq e$ and $\ell\geq1$, denotes the direct edge $X^c(t-\ell)\to X^e(t)$.
Lag-1 self-loops are assumed by the AR(1) model and are not learned; optional C-node edges are stored separately as $G_C[r,e]$ at lag~0.

\cedar{} adapts the two-condition idea of \sypi{}~\citep{mastakouri2021sypi}, which tests whether a candidate $X^c(t-w)\to X^e(t)$ is a direct cause of a designated target.
For each candidate source $c$, \sypi{} conditions on one shifted node from each other candidate time series,
$\mathcal{S}^c=\{X^k(t-w_k-1):k\neq c\}$, and tests
\[
\text{(C1)}\quad X^c(t-w)\not\!\perp\!\!\!\perp X^e(t)\mid\{\mathcal{S}^c,X^e(t-1)\},
\]
\[
\text{(C2)}\quad X^c(t-w-1)\perp\!\!\!\perp X^e(t)\mid\{\mathcal{S}^c,X^c(t-w),X^e(t-1)\}.
\]
Condition~1 establishes remaining dependence, while Condition~2 verifies that the selected lag blocks the influence of the source's deeper past.
This logic relies on strict AR(1) self-dynamics and on the absence of lag-mismatched bypass paths; Appendix~\ref{sec:theory} states the exact sufficient conditions and a counterexample.

For lag screening, \cedar{} uses the signed bias-corrected squared distance correlation $\mathcal{R}^{*2}_n$ of~\citet{szekely2014partial}.
Because its underlying $\mathcal{U}$-centered distance covariance has zero expectation under independence, it provides a fast nonlinear marginal score; definitions are given in Appendix~\ref{sec:dcor_details}.

\section{Method}
\label{sec:method}

\cedar{} proceeds in three main phases: (1) lag selection, (2) CI testing via Conditions~1 and~2, and (3) MCI pruning.
Algorithm~\ref{alg:cedar} gives the full procedure; we describe each component in detail below.

\begin{algorithm}[t]
\caption{\cedar{} summary. Full implementation details follow in Sections~\ref{sec:lag_selection}--\ref{sec:complexity}.}
\label{alg:cedar}
\small
\textbf{Require:} data $\mathbf{X}\in\RR^{T\times d}$, max lag $L$, CI test $\mathcal{T}$, levels $\alpha_1,\alpha_2,\alpha_{\rm lag}$; optional C-node basis $g_1,\ldots,g_R$.
\begin{algorithmic}[1]
  \STATE Optionally append deterministic pure-cause columns $C_r(t)=g_r(t/T)$.
  \STATE Residualize each non-C target: $\tilde X^e(t)=X^e(t)-\hat\beta_eX^e(t-1)$.
  \STATE Score all cross-variable lags with $\mathcal{R}_n^{*2}(X^c(t-\ell),\tilde X^e(t))$; score C-nodes only at lag~0.
  \STATE Apply BH screening at $\alpha_{\rm lag}$.
  \FOR{each target $e$ and candidate source $c$ with screened lag(s)}
    \STATE Build $\mathcal{S}^c$ from other screened non-C candidates at $t-\hat w_{k,e}-1$ and screened C-nodes at lag~0.
    \IF{$c$ is a C-node} \STATE test Condition~1 only and set $G_C[r,e]=1$ if significant.
    \ELSE \STATE test screened lags in decreasing score; accept the first lag passing Conditions~1--2 and set $G_X[c,e,w]=1$.
    \ENDIF
  \ENDFOR
  \STATE Run all MCI pruning tests on the same initial $G_X$ and remove accepted independences simultaneously.
\end{algorithmic}
\end{algorithm}

\subsection{Lag Selection}
\label{sec:lag_selection}

For each variable pair $(X^c, X^e)$ and candidate lag $\ell \in \{1, \ldots, L\}$, we score the lag importance using $\mathcal{U}$-centered distance correlation~\citep{szekely2014partial} on AR(1)-residualized targets.
Specifically, we first remove the target's linear autoregressive component:
\begin{equation}
  \tilde{X}^e(t) = X^e(t) - \hat{\alpha}_e - \hat{\beta}_e \, X^e(t{-}1), \quad
  \hat{\beta}_e = \frac{\sum_t (X^e(t){-}\bar{X}^e)(X^e(t{-}1){-}\bar{X}^e_{-})}{\sum_t (X^e(t{-}1){-}\bar{X}^e_{-})^2}, \quad
  \hat{\alpha}_e = \bar{X}^e - \hat{\beta}_e \bar{X}^e_{-},
  \label{eq:partial_dcor_residual}
\end{equation}
where $\bar{X}^e$ and $\bar{X}^e_{-}$ denote the sample means of $X^e(t)$ and $X^e(t{-}1)$ (i.e., ordinary least squares with an intercept).
Then compute the lag importance score:
\begin{equation}
  M[c, e, \ell] = \mathcal{R}_n^{*2}(X^c(t{-}\ell),\, \tilde{X}^e(t)),
  \label{eq:partial_dcor}
\end{equation}
where $n = T - L$ is the effective sample size.
The residualization prevents the target's AR dynamics from inflating importance scores at irrelevant lags.
Note that this is a simple linear residualization, not partial distance correlation in the Sz\'{e}kely--Rizzo sense~\citep{szekely2014partial}; if the target's autoregressive dynamics are nonlinear, the removal is incomplete.

\paragraph{Analytic $t$-test.}
Since $\EE[\dCov_n^*] = 0$ under independence (Proposition~1 in~\citep{szekely2014partial}), we can exploit the asymptotic null distribution of $\mathcal{R}_n^{*2}$ derived by \citet{szekely2013dcor}.
Define $v = n(n-3)/2$; under $H_0: X \perp\!\!\!\perp Y$, the statistic
\begin{equation}
  T_n = \frac{\sqrt{v-1}\,\cdot\,\mathcal{R}_n^{*2}}{\sqrt{1 - (\mathcal{R}_n^{*2})^2}}
  \label{eq:t_stat}
\end{equation}
approximately follows a Student-$t$ distribution with $v - 1$ degrees of freedom, yielding p-value $p = 1 - F_{t(v-1)}(T_n)$.
This eliminates the need for permutation tests.
The $t$-test assumes i.i.d.\ samples; for strongly autocorrelated data, the effective sample size is smaller than $n$, which can lead to anti-conservative p-values.
All experiments reported here use the analytic test.
Appendix~\ref{sec:calibration} reports calibration under AR(1) nulls: type-I error is mildly inflated (6.4--9.0\% at nominal 5\%), worst at strong autocorrelation and small $T$.

\subsection{Multi-Lag Screening with BH Correction}
\label{sec:multi_lag}

Real causal systems often exhibit effects at multiple lags.
The original \sypi{} framework tests at the minimum lag per pair~\citep{mastakouri2021sypi}, identifying a single causal lag.

\cedar{} applies Benjamini--Hochberg (BH) correction~\citep{benjamini1995controlling} across all $d(d-1)L$ cross-variable lag-significance hypotheses as a multiple-testing screening heuristic.
Let $p_{(1)} \leq p_{(2)} \leq \ldots \leq p_{(m)}$ be the ordered p-values from the lag $t$-tests, with $m = d(d-1)L$ total cross-variable hypotheses.
The BH procedure rejects hypothesis $i$ if:
\begin{equation}
  p_{(i)} \leq \frac{i}{m} \cdot \alpha_{\text{lag}},
  \label{eq:bh}
\end{equation}
where $\alpha_{\text{lag}}$ is a screening threshold (default $0.05$). If C-nodes are used, their lag-0 screening p-values are included in the same screening family unless stated otherwise.

After BH correction, each pair $(X^c, X^e)$ may have multiple significant lags.
We test these in descending order of importance (highest $\mathcal{R}_n^{*2}$ first) and accept only the first lag that passes both Condition~1 and Condition~2 (\textbf{keep=``first''}).
That is, \cedar{} selects one lag per ordered pair; multi-lag screening identifies the best candidate lag rather than recovering all simultaneous lags between a pair.

\subsection{Stable MCI Pruning}
\label{sec:mci}

After the initial CI-based discovery (Phases~1--2), some detected edges may be indirect---mediated through other variables in the discovered graph.
\cedar{} applies a Momentary Conditional Independence (MCI) pruning step inspired by PCMCI~\citep{runge2019detecting} but adapted for the pairwise discovery setting.

For each discovered edge $X^c(t{-}\tau) \to X^e(t)$, we test:
\begin{equation}
  X^c(t{-}\tau) \perp\!\!\!\perp X^e(t) \;\mid\; \mathcal{S}_{\text{MCI}},
  \label{eq:mci_test}
\end{equation}
where the MCI conditioning set combines parents of both the target and the source, each at their respective times:
\begin{equation}
  \mathcal{S}_{\text{MCI}} = \bigl(\text{Pa}(X^e(t)) \setminus \{X^c(t{-}\tau)\}\bigr) \cup \{X^e(t{-}1)\} \cup \text{Pa}(X^c(t{-}\tau)) \cup \{X^c(t{-}\tau{-}1)\} \cup \mathcal{C}_e .
  \label{eq:mci_cond}
\end{equation}
Here $\text{Pa}(X^e(t))$ denotes discovered non-C parents of $X^e$ at their lagged times relative to $t$ (excluding the edge being tested), $X^e(t{-}1)$ is the target's autoregressive term, $\text{Pa}(X^c(t{-}\tau))$ denotes discovered parents of $X^c$ shifted to the source's time, and $\mathcal{C}_e=\{C_r(t):G_C[r,e]=1\}$ when C-nodes are used.
Including $\text{Pa}(X^c(t{-}\tau))$ blocks indirect paths through the source's own parents, reducing false positives from shared ancestry.

\paragraph{Stability (order-independence).}
All MCI tests are performed on the \emph{same} initial graph, and edge removals are applied simultaneously after all tests complete.
This ensures that the pruning result is independent of the order in which edges are tested---a property violated by sequential pruning, where removing one edge changes the conditioning sets for subsequent tests.

\subsection{Nonstationarity via C-Node}
\label{sec:c_node}

Nonstationarity creates spurious associations between variables that share a common trend, leading to false positive edges when ignored.
Following the insight of CD-NOD~\citep{huang2020causal} and \cdnots{}~\citep{sadeghi2025cdnots}, \cedar{} introduces a deterministic time-index variable $C$ that encodes the temporal structure of nonstationarity.

The C-node is a synthetic variable appended to the data matrix before lag selection:
\begin{equation}
  C(t) = g(t / T), \quad t = 1, \ldots, T,
  \label{eq:c_node}
\end{equation}
where $g: [0, 1] \to \RR$ is a deterministic basis function.
When multiple basis functions are used, each becomes a separate C-node column appended to the data.
Available presets include \texttt{linear} ($g(u) = u$, one column), \texttt{linear+quad} ($g_1(u) = u$, $g_2(u) = u^2$, two columns), and \texttt{step+linear} ($g_1(u) = u$, $g_2(u) = \mathbf{1}[u > 0.5]$, two columns).

The C-node is treated as a \emph{pure cause}---it can influence $X$ variables ($C \to X^e$) but is never tested as an effect.
Discovered $C \to X^e$ edges indicate that variable $X^e$'s mechanism is nonstationary.
All $C \to X$ edges are forced to lag~0, reflecting the contemporaneous nature of the nonstationarity influence.
Appendix~\ref{sec:cnode_misspec} shows that \cedar{} degrades gracefully under basis misspecification (F1: $1.00 \to 0.93$); the C-node adjusts for specified deterministic trends rather than guaranteeing robustness to arbitrary nonstationarity.

\subsection{Assumptions}
\label{sec:assumptions}

\cedar{}'s correctness relies on the following assumptions (numbered for cross-reference with Appendix~\ref{sec:theory}):
\begin{enumerate}
  \item[\textup{(A1)}] \textbf{Observed Markovian sufficiency}: The observed time-unrolled process is Markov with respect to its causal DAG, and there are no latent common causes among observed variables. \cedar{}'s MCI pruning (Phase~3) conditions on discovered parents; latent confounders may induce false positive edges that this pruning cannot remove.
  \item[\textup{(A2)}] \textbf{Faithfulness}: Statistical dependencies reflect true causal connections (no exact cancellations).
  \item[\textup{(A3)}] \textbf{Strictly AR(1) self-dynamics}: Every variable $X^j$ has a self-loop $X^j(t{-}1) \to X^j(t)$ (A7 in~\citep{mastakouri2021sypi}) and no self-arrows at lags $s > 1$ (A8). In Appendix~\ref{sec:named_benchmarks}, we test robustness on benchmarks that violate this assumption.
  \item[\textup{(A4)}] \textbf{Stationarity up to specified trend terms}: The cross-variable graph is time-invariant. If C-nodes are used, they model only the specified deterministic trend basis; arbitrary regime changes are outside the oracle guarantee.
\end{enumerate}
Under these assumptions and oracle CI tests, \cedar{}'s pairwise testing correctly identifies direct causes when no alternate directed path from $X^c(t{-}w{-}1)$ to $X^e(t)$ bypasses $X^c(t{-}w)$ while evading the blocking by $\mathcal{S}^c$.
This condition holds for single-lag dependency graphs~\citep{mastakouri2021sypi} (at most one direct lag per ordered pair with no lag-mismatched mediating paths).
We do not claim general oracle correctness; the pairwise conditioning structure is a practical approximation that trades theoretical completeness for statistical power at finite~$T$.
A formal statement with proof sketch and a discussion of the failure mode---where an indirect path $X^c(t{-}w{-}1) \to X^k(t{-}w_k) \to X^e(t)$ evades $\mathcal{S}^c$ because $\mathcal{S}^c$ places $X^k$ at $t{-}w_k{-}1$ rather than $t{-}w_k$---is given in Appendix~\ref{sec:theory}.

\subsection{Complexity}
\label{sec:complexity}

Phase~1 computes $O(d^2 L)$ distance correlations at $O(n \log n)$ each for univariate pairs (AVL-tree algorithm~\citep{huo2016fast}), totaling $O(d^2 L n \log n)$ where $n = T - L$.
Phase~2 performs at most two CI tests per significant-lag candidate; in sparse settings where few lags survive BH screening, this totals approximately $O(d^2)$ tests, though the worst case is $O(d^2 L_{\mathrm{sig}})$ where $L_{\mathrm{sig}}$ is the mean number of significant lags per pair.
Phase~3 performs one MCI test per discovered edge ($|E|$ tests).
With partial correlation, a standard regression-based CI test costs $O(nk^2+k^3)$ for conditioning-set size $k$ (or less with cached/incremental residualization). \cedar{} keeps $k$ small by conditioning only on screened candidates for the same target and on discovered source/target parents during MCI pruning.

\section{Experiments}
\label{sec:experiments}

We evaluate \cedar{} on synthetic benchmarks with known ground truth, comparing against PCMCI+~\citep{runge2020discovering}, \cdnots{}~\citep{sadeghi2025cdnots}, and \cdnots{}+ (which replaces \cdnots{}'s PC-style skeleton discovery with PCMCI+-style iterative conditioning).
All methods support arbitrary CI tests. We use partial correlation (ParCorr) with $\alpha=0.01$ in the large-scale experiments because it is data-efficient and fast at high $d$; Appendix~\ref{sec:named_benchmarks} demonstrates pairing \cedar{} with SplitKCI on a nonlinear benchmark.
\cedar{} uses AR(1)-residualized dcor for lag selection (\texttt{partial\_dcor}) by default.
The \textsc{Causal-TS} package contains the implementation, experiment scripts, and seed-level outputs.
Additional results on fixed-structure benchmarks are in Appendix~\ref{sec:named_benchmarks}.

\subsection{Scaling with Dimensionality and Sample Size}
\label{sec:scaling}

Figure~\ref{fig:f1_vs_d} evaluates \cedar{} on random sparse structural causal processes (SCPs) with increasing dimensionality at three sample sizes: $T{=}100$ (very data-scarce), $T{=}200$ (data-scarce), and $T{=}500$ (data-rich).
All F1 scores are lag-specific: an edge $(X^c(t{-}\tau) \to X^e(t))$ counts as a true positive only if both the variable pair and the exact lag $\tau$ match the ground truth.
Edge functions are drawn from a mixed pool: approximately half linear and half nonlinear (quadratic and sinusoidal).
The comparison reveals a clear sample-size effect.
At $T \leq 200$, \cedar{} is the best or tied-best method across most configurations---particularly on scale-free and small-world topologies for $d \leq 40$, where it leads by 1--4 F1 points.
This advantage stems from \cedar{}'s fixed, small conditioning sets: with at most two CI tests per significant lag candidate, each test retains high statistical power even when $T/d$ is small.
At $T{=}500$, methods with richer conditioning sets (PCMCI+, \cdnots{}+) close the gap and often surpass \cedar{}, especially on scale-free graphs where hub structure benefits from conditioning on larger parent sets.
Full numerical results (mean and standard deviation over 20 seeds) for all sample sizes are reported in Appendix~\ref{sec:full_results}.

\begin{figure}[t]
\centering
\includegraphics[width=\textwidth]{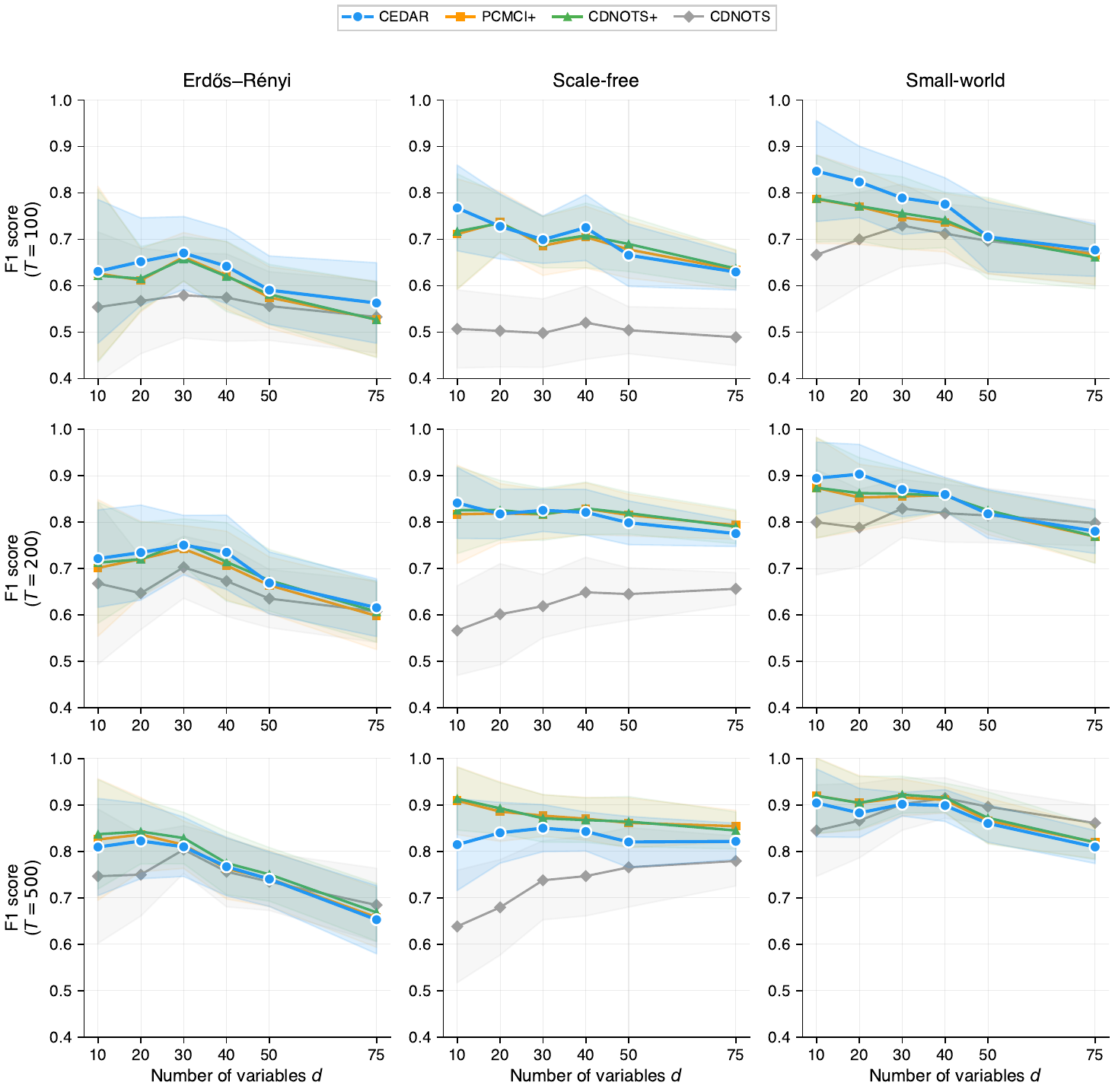}
\caption{F1 score vs.\ dimensionality $d$ for three graph topologies and three sample sizes. Shaded bands show $\pm 1$ standard deviation over 20 seeds.}
\label{fig:f1_vs_d}
\end{figure}

\subsection{Robustness and Failure Modes}
\label{sec:robustness}

Because \cedar{}'s oracle result relies on strict AR(1) self-dynamics (A3), a single true direct lag per pair (A5), and the path-blocking condition (A7), we explicitly evaluate violations of these assumptions.
Table~\ref{tab:stress_main} reports \cedar{} and \cdnots{}+ on three stress tests with graphs large enough to expose genuine failure modes ($d{=}15$--$20$, $T{=}300$, 10 seeds).

\begin{table}[t]
\centering
\caption{Robustness under assumption violations. \cdnots{}+ is more robust due to its richer conditioning sets, but \cedar{} degrades gracefully---F1 remains between 0.865 and 0.924 even under simultaneous violations. The gap grows predictably as violations accumulate.}
\label{tab:stress_main}
\small
\begin{tabular}{l l cc}
\toprule
\textbf{Violation} & \textbf{Method} & Prec & F1 \\
\midrule
\multirow{2}{*}{A7 bypass $\times$5 ($d{=}15$, A7)} & \cedar{} & .893 & .924 \\
 & \cdnots{}+ & .934 & .965 \\
\midrule
\multirow{2}{*}{Dense AR(2) ($d{=}15$, A3)} & \cedar{} & .875 & .886 \\
 & \cdnots{}+ & .934 & .963 \\
\midrule
\multirow{2}{*}{Mixed A3+A5+A7 ($d{=}20$)} & \cedar{} & .905 & .865 \\
 & \cdnots{}+ & .958 & .979 \\
\bottomrule
\end{tabular}
\end{table}

\cdnots{}+ is more robust under these violations, as expected from its iterative conditioning.
However, \cedar{} remains competitive---F1 between 0.865 and 0.924---indicating graceful rather than catastrophic degradation.
These results clarify the intended use case: \cedar{} is preferable in data-scarce settings where variables exhibit lag-1 self-dynamics, whereas \cdnots{}+/PCMCI+ are preferable when $T$ is large or assumption violations are expected.
Full TP/FP/FN breakdowns are in Appendix~\ref{sec:stress_tests}.

\subsection{Ablation Study}
\label{sec:ablation}

Table~\ref{tab:ablation} shows the incremental contribution of each \cedar{} feature, starting from a \sypi{}-style baseline (Pearson correlation, single-lag, no pruning).
The overall trajectory is strongly positive, but intermediate steps are non-monotone: adding the analytic $t$-test and multi-lag selection temporarily reduces precision because both increase recall (more lags are tested and more edges proposed) before MCI pruning corrects the false positives.
MCI pruning is load-bearing: it recovers the precision gain, confirming that lag selection and CI testing work in concert with the pruning step.
The final row replaces plain dcor with AR(1)-residualized dcor in the lag selection phase---this is not an additional step after MCI pruning but a modification of Phase~1 that provides the largest single improvement, demonstrating that removing the target's autoregressive component from lag importance scores is critical.
Appendix~\ref{sec:ablation_topologies} confirms the same pattern holds across Erd\H{o}s--R\'{e}nyi, scale-free, and small-world topologies.

To directly establish novelty over the closest ancestor, Appendix~\ref{sec:sypi_comparison} compares \cedar{} against a faithful \sypi{}-all-targets adaptation: LassoCV lag selection (matching the original paper's lag identification), $\alpha_1{=}\alpha_2{=}0.05$ (from the paper's threshold sweep), and no MCI pruning.
\cedar{} outperforms this baseline by 0.18--0.46 F1 points across all settings.
The gap is explained by two independent factors: (i)~AR(1)-residualized dcor detects lag associations more precisely than Lasso, especially on scale-free graphs; (ii)~MCI pruning corrects the high false-positive rate that results from Lasso's simultaneous fit of all lags without the AR decomposition (mean precision of \sypi{}-all-targets: 0.34 vs.\ 0.87 for \cedar{}).
\cedar{} is also the fastest method at every $d$ tested ($2.4\times$ faster than PCMCI+ at $d{=}75$); runtime details are in Appendix~\ref{sec:full_results}, Table~\ref{tab:runtime}.

\begin{table}[t]
\centering
\caption{Incremental ablation on the 11-node cascade network ($d{=}11$, $\alpha{=}0.01$, mean over 10 seeds). Each row adds one feature to the previous configuration. The pattern is consistent across sample sizes: AR(1)-residualized dcor and MCI pruning are the key components.}
\label{tab:ablation}
\small
\begin{tabular}{l ccc ccc ccc}
\toprule
& \multicolumn{3}{c}{$T{=}100$} & \multicolumn{3}{c}{$T{=}200$} & \multicolumn{3}{c}{$T{=}500$} \\
\cmidrule(lr){2-4} \cmidrule(lr){5-7} \cmidrule(lr){8-10}
\textbf{Configuration} & F1 & Prec & Rec & F1 & Prec & Rec & F1 & Prec & Rec \\
\midrule
Baseline (\sypi{}-style) & .447 & .559 & .376 & .696 & .734 & .676 & .731 & .679 & .794 \\
\quad + $\mathcal{U}$-centered dcor & .464 & .585 & .388 & .683 & .731 & .659 & .722 & .672 & .782 \\
\quad + Analytic $t$-test & .585 & .703 & .506 & .679 & .685 & .682 & .710 & .644 & .794 \\
\quad + Multi-lag & .623 & .696 & .571 & .699 & .678 & .729 & .704 & .622 & .818 \\
\quad + MCI pruning & .591 & .821 & .465 & .763 & .860 & .694 & .822 & .828 & .818 \\
\quad + AR(1)-res.\ dcor (\textbf{\cedar{}}) & \textbf{.700} & \textbf{.958} & \textbf{.553} & \textbf{.887} & \textbf{.970} & \textbf{.824} & \textbf{.964} & \textbf{.971} & \textbf{.959} \\
\bottomrule
\end{tabular}
\end{table}

\subsection{Real-World Application: River Network Discovery}
\label{sec:river}

We evaluate \cedar{} on the Elbe River benchmark~\citep{stein2025causalrivers}: 12 gauging stations on the Elbe main branch ($T{=}11{,}688$, $d{=}12$, 3-hour resolution); ground truth is a chain of 11 edges (upstream $\to$ downstream).
The data is challenging because hydrological regimes change the causal lag structure and rainfall creates transient spurious correlations (see Appendix~\ref{sec:elbe_details} for full description).
We apply PELT changepoint detection to segment into three regimes (low/normal/high-flow) and run \cedar{} per regime with adapted lag bounds ($L{=}8,5,3$), taking the union of discovered edges.

\begin{figure}[t]
\centering
\includegraphics[width=\textwidth]{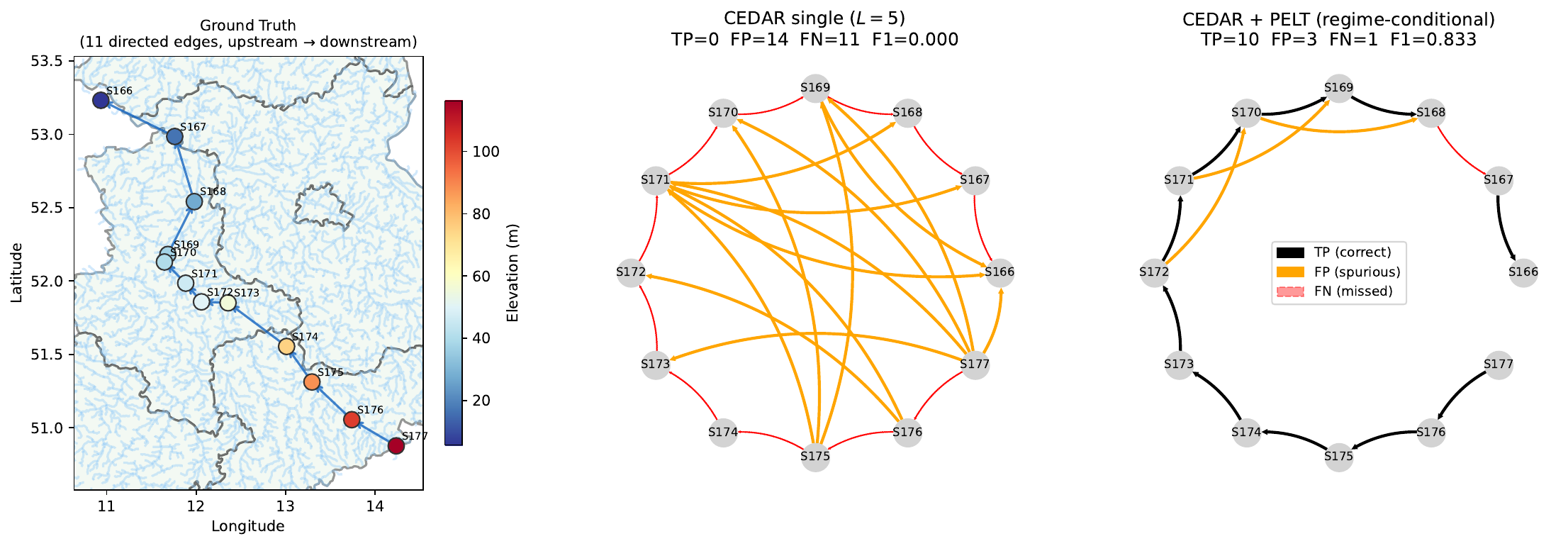}
\caption{Elbe River main branch ($d{=}12$, 11 true edges). \textbf{Left:} ground truth chain. \textbf{Center:} CEDAR single run---nonstationarity confounds discovery (0 TP, 14 FP). \textbf{Right:} CEDAR + PELT---low-flow regime exposes clean causal signal, recovering 10 of 11 edges (F1$=$0.833). Green: TP, orange: FP, red dashed: FN.}
\label{fig:elbe_compare}
\end{figure}

Regime-conditional \cedar{} achieves F1$=$0.833 (10 TP, 3 FP) vs.\ F1$=$0.000 on the single run (Figure~\ref{fig:elbe_compare}), with the low-flow regime alone yielding 0.900 precision (Appendix~\ref{sec:elbe_details}, Table~\ref{tab:elbe}).
\cdnots{}+ similarly benefits from regime conditioning (F1: 0.417 $\to$ 0.750), though \cedar{} outperforms it in this data-scarce-per-segment setting.

\section{Conclusion}
\label{sec:conclusion}

We presented \cedar{}, a training-free method for lagged causal edge discovery in sparse autoregressive time series. After lag screening, \cedar{} requires $O(d^2)$ CI tests in sparse regimes and empirically extends \sypi{}'s two-condition test from single-target feature selection toward full graph discovery under additional structural restrictions (Appendix~\ref{sec:theory}).
The method addresses three limitations of the original \sypi{} framework: single-lag selection is replaced by AR(1)-residualized dcor with BH-corrected multi-lag testing; the sink-node restriction is relaxed via temporal acyclicity by testing all ordered pairs and applying stable MCI pruning; and a deterministic C-node reduces trend-induced false positives.

Together these components yield a method that excels in the data-scarce regime: at $T \leq 200$, \cedar{} is the best-performing method in the majority of tested configurations across three graph topologies and dimensionalities up to $d{=}75$.
As $T$ increases ($T \geq 500$), the trade-off reverses: richer conditioning sets become affordable, allowing PCMCI+ and \cdnots{}+ to surpass \cedar{} on scale-free and dense graphs.
On a real-world river network extracted from CausalRivers (Section~\ref{sec:river}), regime-conditional \cedar{} with PELT segmentation recovers 10 of 11 true edges (F1$=$0.833), with the low-flow regime yielding 0.900 precision---an encouraging result on a real-world nonstationary system despite missing data and known violation of causal sufficiency.
It uses standard significance levels with few hyperparameters ($\alpha$, $L$, C-node basis), making it practical when data are limited.

\paragraph{Limitations.}
The two-condition CI tests assume each variable has a lag-1 self-dependency ($X^j(t{-}1) \to X^j(t)$); for variables with higher-order autoregressive dynamics, unconditioned lags may weaken both conditions, producing false negatives that MCI pruning only partially recovers.
The analytic $t$-test for lag selection assumes i.i.d.\ samples and becomes anti-conservative under strong autocorrelation (Appendix~\ref{sec:calibration}); more extensive calibration under varying network structures and non-Gaussian noise, together with dependent-data corrections, is left for future work.
The C-node approach benefits from a correctly specified basis; misspecification reduces but does not eliminate its effectiveness (Appendix~\ref{sec:cnode_misspec}).
At large $T$ ($\geq 500$), methods with iterative conditioning (PCMCI+, \cdnots{}+) achieve higher recall by conditioning on richer parent sets---a trade-off inherent to the pairwise design.

\bibliographystyle{plainnat}
\bibliography{references}

\appendix
\section{Theoretical Analysis}
\label{sec:theory}

\paragraph{Scope of the oracle guarantee.}
The correctness result applies only to eligible cross-variable lagged edges under observed Markovian sufficiency, faithfulness, strict AR(1) self-dynamics, single recovered lag per ordered pair, correct lag ordering, and the CEDAR path-blocking condition.
It does not cover C-node edges, latent confounders, contemporaneous effects, changing graph structure, or arbitrary multi-lag dynamics.

We state and prove \cedar{}'s correctness guarantees in the oracle setting.
Let $\mathcal{G}^*$ be the time-unrolled DAG over nodes $X_i^s$ ($i \in \{1,\ldots,d\}$, $s \in \mathbb{Z}$).
A lagged edge $X^c(t{-}w) \to X^e(t)$ means the DAG contains $X_c^{t-w} \to X_e^t$.
For a candidate edge $X^c(t{-}w) \to X^e(t)$, write $P := X_c^{t-w}$, $D := X_c^{t-w-1}$, $Y := X_e^t$, and define the Phase-2 conditioning sets:
\begin{align*}
  \mathcal{S}_{e,c} &:= \{X_k^{t-\hat{w}_{k,e}-1} : k \neq c,\, \hat{w}_{k,e} \text{ exists}\}, \\
  Z_1 &:= \mathcal{S}_{e,c} \cup \{X_e^{t-1}\}, \\
  Z_2 &:= Z_1 \cup \{P\},
\end{align*}
where $\hat{w}_{k,e}$ is the Phase-1 selected lag of $X^k$ toward $X^e$.
The MCI conditioning set for edge $P \to Y$ is
\[
  \mathcal{S}_{\mathrm{MCI}}(P,Y) :=
  \bigl(\widehat{\mathrm{Pa}}^{(2)}_\times(Y) \setminus \{P\}\bigr)
  \cup \{X_e^{t-1}\}
  \cup \widehat{\mathrm{Pa}}^{(2)}_\times(P)
  \cup \{X_c^{t-\tau-1}\}.
\]
Here $\widehat{\mathrm{Pa}}^{(2)}_\times$ denotes Phase-2 discovered cross-variable parents; self-loops are treated separately through the fixed AR(1) terms.

\paragraph{Assumptions.}
Assumptions A1--A4 are stated in Section~\ref{sec:assumptions}: (A1)~observed Markovian sufficiency, taken to include strict forward-time ordering ($X_i^s \to X_j^r \Rightarrow s<r$); (A2)~full faithfulness ($\perp\!\!\!\perp \Leftrightarrow$ d-separation); (A3)~strictly AR(1) self-dynamics; and (A4)~stationarity up to specified trend terms (used by the empirical lag screening, not by the pointwise d-separation arguments below).
The following additional assumptions are specific to the correctness analysis:
\begin{enumerate}
  \item[\textup{(A5)}] \textbf{Single true direct lag per ordered pair.} For each ordered pair $(c,e)$, $c \neq e$, at most one lag satisfies $G^*[c,e,w]=1$. This is a restriction on the true graph; that \cedar{} outputs at most one lag per pair follows separately from the first-accept rule.
  \item[\textup{(A6)}] \textbf{Oracle candidate reachability and ordering.} For every true edge $X^c(t{-}w) \to X^e(t)$, lag $w$ appears in the ordered candidate list supplied to Phase~2, and every candidate appearing before $w$ fails at least one of Conditions~1--2 under oracle CI, so the first-accept rule reaches $w$. This is an algorithm-dependent oracle success condition, not a consequence of A1--A3.
  \item[\textup{(A7)}] \textbf{CEDAR path-blocking.} For every true edge $X^c(t{-}w) \to X^e(t)$, the deeper source node $D = X_c^{t-w-1}$ is d-separated from $Y = X_e^t$ by the Condition-2 set in $\mathcal{G}^*$:
  \[
    X_c^{t-w-1} \perp_{\mathcal{G}^*} X_e^t \;\mid\; Z_2.
    \tag{A-CB}
  \]
  Equivalently, \emph{every} path from $X_c^{t-w-1}$ to $X_e^t$ is blocked by $Z_2$. It does not suffice to check only paths that avoid $P = X_c^{t-w}$: since $P \in Z_2$, conditioning on $P$ blocks a path on which $P$ is a non-collider but \emph{opens} a path on which $P$ is a collider, so any such collider path must be blocked elsewhere.
  \item[\textup{(A8)}] \textbf{Eligible-parent closure.} For every target $Y = X_e^t$ covered by the theorem, every true parent of $Y$ is either the lag-1 self-parent $X_e^{t-1}$ or a true eligible cross-variable parent $X_c^{t-w}$ with $(c,e,w) \in \mathcal{E}_{\mathrm{elig}}$, where $\mathcal{E}_{\mathrm{elig}}$ denotes the positive-lag, cross-variable, non-C candidate triples that \cedar{} searches.
\end{enumerate}

\begin{proposition}[Phase-2 exact criterion]
\label{prop:phase2}
Assume~A1--A2. Let $X_c^{t-w} \to X_e^t$ be a true direct edge with $c \neq e$; by construction $P \notin Z_1$, $Y \notin Z_2$, and $D \notin Z_2$ (the source variable $c$ is excluded from $\mathcal{S}_{e,c}$, and all of $Z_2$ lies at times $\leq t-1$). Then:
\begin{enumerate}
  \item[\textup{(i)}] Condition~1 always passes under oracle CI\@.
  \item[\textup{(ii)}] Condition~2 passes under oracle CI if and only if \textup{(A-CB)} holds.
\end{enumerate}
\end{proposition}

\begin{proof}
\textit{(i).}
The direct path $P \to Y$ has no internal nodes and is d-connected given $Z_1$ since neither endpoint is in $Z_1$.
By faithfulness, $P \not\!\perp\!\!\!\perp Y \mid Z_1$, so Condition~1 rejects independence.

\textit{(ii).}
Condition~2 tests $D \perp\!\!\!\perp Y \mid Z_2$.
Under the Markov and faithfulness assumptions, this oracle CI statement holds if and only if $D$ and $Y$ are d-separated by $Z_2$ in $\mathcal{G}^*$, which is exactly \textup{(A-CB)}.
\end{proof}

\begin{proposition}[Phase-2 completeness and correct lag on true pairs]
\label{prop:completeness}
Assume A1--A2 and A5--A7. Then every true eligible cross-variable parent is selected by Phase~2 at its correct lag: if $G^*[c,e,w]=1$ for an eligible triple $(c,e,w)$, then $X_c^{t-w} \in \widehat{\mathrm{Pa}}^{(2)}_\times(X_e^t)$.
\end{proposition}

\begin{proof}
Fix a true eligible edge $X_c^{t-w} \to X_e^t$.
By A6, the true lag $w$ appears in the candidate list and no earlier candidate passes both conditions, so the first-accept rule reaches $w$.
By Proposition~\ref{prop:phase2}(i), Condition~1 passes at lag $w$; by A7 and Proposition~\ref{prop:phase2}(ii), Condition~2 passes.
Hence Phase~2 accepts $X_c^{t-w} \to X_e^t$. By A5 there is no second true direct lag for the pair, so the edge is selected at the correct lag.
\end{proof}

\begin{proposition}[Stable MCI pruning is exact]
\label{prop:mci}
Assume A1--A2. Suppose that after Phase~2, $\mathrm{Pa}^*(Y) \subseteq \widehat{\mathrm{Pa}}^{(2)}_\times(Y) \cup \{X_e^{t-1}\}$ for every target $Y = X_e^t$ to which MCI is applied (denoted \textup{(A-PC)}).
Then, under oracle CI tests, stable MCI pruning removes every false positive and preserves every true eligible cross-variable edge among the edges it tests.
\end{proposition}

\begin{proof}
\textit{False positive.}
Let $P = X_c^{t-\tau} \notin \mathrm{Pa}^*(Y)$.
Because all cross-variable edges have lag $\geq 1$, every variable in $\mathcal{S}_{\mathrm{MCI}}(P,Y)$ lies at a time strictly before $t$; by temporal acyclicity, every such variable is a non-descendant of $Y$.
By assumption, $\mathrm{Pa}^*(Y) \subseteq \widehat{\mathrm{Pa}}^{(2)}_\times(Y) \cup \{X_e^{t-1}\} \subseteq \mathcal{S}_{\mathrm{MCI}}(P,Y) \cup \{P\}$.
Since $P \notin \mathrm{Pa}^*(Y)$, this gives $\mathrm{Pa}^*(Y) \subseteq \mathcal{S}_{\mathrm{MCI}}(P,Y)$.
By the local Markov property applied to $Y$:
\[
  Y \perp_{\mathcal{G}^*} \bigl(\mathrm{ND}(Y) \setminus \mathrm{Pa}^*(Y)\bigr) \mid \mathrm{Pa}^*(Y),
\]
where all variables in $\mathcal{S}_{\mathrm{MCI}} \setminus \mathrm{Pa}^*(Y)$ and $P$ itself are in $\mathrm{ND}(Y)$ by temporal acyclicity.
By decomposition, restricting to the subset $P \cup (\mathcal{S}_{\mathrm{MCI}} \setminus \mathrm{Pa}^*(Y)) \subseteq \mathrm{ND}(Y) \setminus \mathrm{Pa}^*(Y)$, we obtain $Y \perp_{\mathcal{G}^*} (P \cup (\mathcal{S}_{\mathrm{MCI}} \setminus \mathrm{Pa}^*(Y))) \mid \mathrm{Pa}^*(Y)$.
By weak union, $Y \perp_{\mathcal{G}^*} P \mid \mathcal{S}_{\mathrm{MCI}}(P,Y)$.
The oracle CI test accepts this independence, so MCI removes the false positive.
This does not rely on the (generally false) claim that conditioning on non-descendants never opens a path---it can open a collider path. What guarantees the independence is the joint local-Markov statement together with the inclusion of every true parent of $Y$ in $\mathcal{S}_{\mathrm{MCI}}$, via decomposition and weak union.

\textit{True edge.}
Let $P \to Y$ be a true edge. The direct path $P \to Y$ has no internal nodes, and neither $P$ nor $Y$ is in $\mathcal{S}_{\mathrm{MCI}}(P,Y)$ (by construction, $P$ is excluded from the target-parent component).
The direct path is therefore d-connected given $\mathcal{S}_{\mathrm{MCI}}(P,Y)$.
By faithfulness, $P \not\!\perp\!\!\!\perp Y \mid \mathcal{S}_{\mathrm{MCI}}(P,Y)$, so the oracle CI test rejects independence and the true edge is preserved.
\end{proof}

\begin{theorem}[Conditional oracle exactness on the eligible edge domain]
\label{thm:correctness}
Assume A1--A8 and oracle CI tests, and that every Phase-2 accepted eligible edge is tested by stable MCI using the frozen Phase-2 parent sets. Then \cedar{} exactly recovers the eligible positive-lag cross-variable graph:
\[
  \widehat{G}_X[c,e,w] = G^*[c,e,w] \qquad \text{for every } (c,e,w) \in \mathcal{E}_{\mathrm{elig}}.
\]
No claim is made for triples outside $\mathcal{E}_{\mathrm{elig}}$ (self-loops are assumed rather than learned, and C-node edges are excluded).
\end{theorem}

\begin{proof}
Fix a target $Y = X_e^t$. By A8, $\mathrm{Pa}^*(Y) \subseteq \{X_e^{t-1}\} \cup \{X_c^{t-w} : (c,e,w) \in \mathcal{E}_{\mathrm{elig}},\; G^*[c,e,w]=1\}$. By Proposition~\ref{prop:completeness}, every true eligible cross-variable parent of $Y$ lies in $\widehat{\mathrm{Pa}}^{(2)}_\times(Y)$, so $\mathrm{Pa}^*(Y) \subseteq \widehat{\mathrm{Pa}}^{(2)}_\times(Y) \cup \{X_e^{t-1}\}$; the premise \textup{(A-PC)} of Proposition~\ref{prop:mci} holds.
For a pair with a true eligible edge at lag $w$: Proposition~\ref{prop:completeness} selects it at $w$, the first-accept rule adds no second lag (A5), and Proposition~\ref{prop:mci} preserves it, while every other eligible lag $u \neq w$ has $G^*[c,e,u]=0$ and yields no surviving edge. For a pair with no true eligible edge: any Phase-2 acceptance is a false positive and is removed by Proposition~\ref{prop:mci}. Hence $\widehat{G}_X[c,e,w] = G^*[c,e,w]$ for all $(c,e,w) \in \mathcal{E}_{\mathrm{elig}}$.
\end{proof}

\noindent
A6 and A7 are strong, method-specific oracle conditions rather than consequences of A1--A3: A6 requires the first-accept order to reach the true lag, and A7 requires the deeper source to be d-separated from the target by the conditioning set \cedar{} constructs.
As the counterexample below shows, A7 fails when lag-mismatched indirect paths exist, limiting the theorem to graphs where causal lags and mediating paths are aligned.

\paragraph{Sink-node relaxation.}
The original \sypi{} required assumption~A6$_{\text{SyPI}}$ (the target is a sink node) to prevent conditioning on descendants of the target.
In \cedar{}, all variables in $Z_1 = \mathcal{S}_{e,c} \cup \{X_e^{t-1}\}$ lie at times strictly before $t$; by strict forward-time ordering (A1), no variable in $Z_1$ is a descendant of $X_e^t$.
Hence the sink assumption is unnecessary \emph{for the specific purpose of ensuring that $Z_1$ contains no descendants of the target}. This does not by itself make arbitrary conditioning harmless---earlier nodes can still be colliders on other paths---which is why the separate path-blocking condition A7 is still required for Condition~2.

\paragraph{Why assumption A7 is necessary.}
A1--A5 alone do not suffice for Phase-2 completeness, even with at most one direct lag per ordered pair.

\textbf{Counterexample.}
Consider three processes $A, B, C$ with first-order self-dynamics and the stationary cross-edge templates
\[
  A^{s-2} \to C^s, \qquad A^{s-2} \to B^s, \qquad B^{s-1} \to C^s \quad \text{for all } s,
\]
i.e.\ $A \to C$ and $A \to B$ at lag~2 and $B \to C$ at lag~1, together with self-loops $A^{s-1}\to A^s$, $B^{s-1}\to B^s$, $C^{s-1}\to C^s$. A concrete stable realization is
\[
  A_s = a A_{s-1} + \varepsilon^A_s,\quad
  B_s = b B_{s-1} + u A_{s-2} + \varepsilon^B_s,\quad
  C_s = c C_{s-1} + v A_{s-2} + q B_{s-1} + \varepsilon^C_s,
\]
with mutually independent nondegenerate noises and generic nonzero coefficients $u,v,q$ chosen so the process is stable and faithful. There are no latent confounders or contemporaneous edges, and each ordered pair has at most one direct lag.
Assume Phase~1 includes lag~2 for $A \to C$, no earlier $A \to C$ candidate passes both conditions, and it supplies $\hat{w}_{B,C}=1$.

When testing the true edge $A^{t-2} \to C^t$ we have $P = A^{t-2}$, $D = A^{t-3}$, $Y = C^t$, and since $\hat{w}_{B,C}=1$, $\mathcal{S}_{C,A} = \{B^{t-2}\}$, so
\[
  Z_2 = \{B^{t-2},\, A^{t-2},\, C^{t-1}\}.
\]
Taking $s = t-1$ in the lag-2 template $A^{s-2} \to B^s$ gives $A^{t-3} \to B^{t-1}$, so the graph contains the path $A^{t-3} \to B^{t-1} \to C^t$. Its only internal node $B^{t-1}$ is a non-collider and is not in $Z_2$, and no conditioned node lies on it; the path is therefore active given $Z_2$, so $A^{t-3} \not\perp_{\mathcal{G}^*} C^t \mid Z_2$. By faithfulness \textup{(A-CB)} fails, Condition~2 rejects independence, and \cedar{} misses the true edge $A^{t-2} \to C^t$.

\noindent
\textbf{Why $B^{t-2}$ does not block the bypass.}
$\mathcal{S}_{e,c}$ places $X^k$ at time $t{-}\hat{w}_{k,e}{-}1$, one step \emph{before} its direct-cause time.
The conditioned node $B^{t-2}$ would block a route entering $B$ through $B^{t-2}$ and its self-loop $B^{t-2} \to B^{t-1}$; but the lag-2 edge $A \to B$ injects the deeper source directly at $B^{t-1}$ via $A^{t-3} \to B^{t-1}$, bypassing $B^{t-2}$.
The example satisfies A1--A6 and A8 but violates A7.

\section{Distance-Correlation Details}
\label{sec:dcor_details}

For a sample $(X_1,\ldots,X_n)$, $n>2$, let $a_{kl}=\|X_k-X_l\|$ and define the $\mathcal{U}$-centered distance matrix~\citep[Def.~2, Eq.~3.1]{szekely2014partial}
\begin{equation}
  \widetilde{A}_{kl}=\begin{cases}
    a_{kl}-\dfrac{1}{n-2}\sum_{s=1}^n a_{ks}-\dfrac{1}{n-2}\sum_{s=1}^n a_{sl}+\dfrac{1}{(n-1)(n-2)}\sum_{s,t=1}^n a_{st}, & k\neq l,\\[6pt]
    0, & k=l .
  \end{cases}
\end{equation}
Given the analogous matrix $\widetilde{B}$ for $Y$, define
\[
  \dCov_n^*(X,Y)=\frac{1}{n(n-3)}\sum_{k\neq l}\widetilde{A}_{kl}\widetilde{B}_{kl},\qquad
  \mathcal{R}_n^{*2}(X,Y)=\frac{\dCov_n^*(X,Y)}{\sqrt{\dCov_n^*(X,X)\dCov_n^*(Y,Y)}}.
\]
The estimator $\dCov_n^*$ has zero expectation under independence for all $n$; consequently, $\mathcal{R}_n^{*2}$ is signed and can be negative under the null.

\section{Calibration of the Analytic $t$-Test}
\label{sec:calibration}

Table~\ref{tab:calibration} reports the empirical type-I error of the analytic distance-correlation $t$-test under null AR(1) graphs (no cross-variable links, $d{=}5$, 200 simulations per setting, nominal $\alpha{=}0.05$).
The test is mildly anti-conservative, with inflation increasing under stronger autocorrelation and smaller sample sizes.
At moderate autocorrelation (AR${\leq}0.5$), type-I error remains below 7.5\%; at very strong autocorrelation (AR${=}0.95$) it reaches ${\sim}9\%$ for $T{=}100$.
Note that even the i.i.d.\ rows (AR${=}0$) sit at ${\sim}6.5\%$ for every $T$: the analytic $t$-approximation of \citet{szekely2013dcor} is derived in the high-dimensional limit, so on the scalar series tested here it is mildly anti-conservative regardless of autocorrelation; the AR(1) dependence adds a further, sample-size-dependent inflation on top of this floor.
The AR(1) residualization in \cedar{}'s lag selection partially mitigates this inflation by removing the dominant autocorrelation before computing distance correlation.

\begin{table}[h]
\centering
\caption{Empirical type-I error of the analytic dcor $t$-test at nominal $\alpha{=}0.05$, under AR(1) null (no cross-links, $d{=}5$, 200 simulations).}
\label{tab:calibration}
\small
\begin{tabular}{l ccc}
\toprule
\textbf{AR coeff.} & $T{=}100$ & $T{=}200$ & $T{=}500$ \\
\midrule
0.00 (i.i.d.) & .067 & .064 & .065 \\
0.20           & .067 & .067 & .066 \\
0.50           & .072 & .068 & .066 \\
0.80           & .075 & .076 & .069 \\
0.95           & .090 & .084 & .072 \\
\bottomrule
\end{tabular}
\end{table}

\section{C-Node Misspecification}
\label{sec:cnode_misspec}

Table~\ref{tab:cnode_misspec} evaluates robustness to C-node basis misspecification on the confounded-trends benchmark ($T{=}500$, 10 seeds).
The data contains a shared quadratic trend; the correct basis is \texttt{linear+quad}.
\cedar{} degrades minimally with a wrong or absent C-node (F1: $1.00 \to 0.93$), because its MCI pruning naturally removes trend-induced indirect edges.
In contrast, \cdnots{} and \cdnots{}+ rely on the C-node for Phase~3 orientation and collapse when the basis is misspecified (F1: $0.92 \to 0.13$).
PCMCI+ lacks C-node support entirely and performs poorly regardless.

\begin{table}[h]
\centering
\caption{C-node misspecification on confounded-trends ($T{=}500$, $L{=}2$, mean F1 over 10 seeds). \texttt{linear+quad} is the correct basis. Cross-var = cross-variable edges only; All = including self-loops. \cedar{} assumes AR(1) self-loops for all variables by design; variables without true self-lags produce self-loop FPs in the ``All'' metric.}
\label{tab:cnode_misspec}
\small
\begin{tabular}{l cc cc cc}
\toprule
& \multicolumn{2}{c}{\texttt{linear+quad}} & \multicolumn{2}{c}{\texttt{linear} (wrong)} & \multicolumn{2}{c}{No C-node} \\
\cmidrule(lr){2-3} \cmidrule(lr){4-5} \cmidrule(lr){6-7}
\textbf{Method} & Cross & All & Cross & All & Cross & All \\
\midrule
\textbf{\cedar{}}  & \textbf{1.000} & .769 & \textbf{.933} & \textbf{.758} & \textbf{.933} & \textbf{.758} \\
\cdnots{}          & .900 & \textbf{.948} & .123 & .391 & .118 & .376 \\
\cdnots{}+         & .800 & .930 & .237 & .558 & .335 & .643 \\
PCMCI+             & .348 & .675 & .348 & .675 & .348 & .675 \\
\bottomrule
\end{tabular}
\end{table}

\section{Fixed-Structure Benchmarks}
\label{sec:named_benchmarks}

Table~\ref{tab:main} reports results on four benchmarks from the \textsc{Causal-TS} package (available via \texttt{load\_dataset()}).
We describe each benchmark's structural equations below.

\paragraph{Cascade Network, 11-Node.}
Inspired by the protein signaling cascade of~\citet{sachs2005causal}, with hub nodes (PKC, PKA) driving a downstream cascade. $L{=}3$, $\varepsilon_i \sim \mathcal{N}(0,1)$.

\noindent
\begin{minipage}[t]{0.52\textwidth}
\small
\begin{align*}
X_0^t &= 0.9\,X_0^{t-1} + \varepsilon_0 \\
X_1^t &= 0.6\,X_1^{t-1} + 0.5\,X_0^{t-2} + \varepsilon_1 \\
X_2^t &= 0.4\,X_2^{t-1} + 0.4\,X_0^{t-3} + 0.3\,X_1^{t-1} + \varepsilon_2 \\
X_3^t &= 0.25\,X_3^{t-1} + \varepsilon_3 \\
X_4^t &= 0.25\,X_4^{t-1} - 0.25\,X_3^{t-1} + \varepsilon_4 \\
X_5^t &= 0.25\,X_5^{t-1} + 0.5\,X_3^{t-2} - 0.5\,X_4^{t-1} + \varepsilon_5 \\
X_6^t &= 0.3\,X_6^{t-1} - 0.25\,X_3^{t-3} + 0.25\,X_4^{t-3} + \varepsilon_6 \\
X_7^t &= 0.4\,X_7^{t-1} + 0.35\,X_3^{t-1} + 0.25\,X_4^{t-2} + \varepsilon_7 \\
X_8^t &= 0.5\,X_8^{t-1} - 0.5\,X_3^{t-1} + 0.25\,X_4^{t-1} + 0.6\,X_7^{t-2} + \varepsilon_8 \\
X_9^t &= 0.6\,X_9^{t-1} + 0.45\,X_4^{t-1} - 0.35\,X_8^{t-2} + \varepsilon_9 \\
X_{10}^t &= 0.7\,X_{10}^{t-1} - 0.25\,X_4^{t-1} + 0.25\,X_9^{t-2} + \varepsilon_{10}
\end{align*}
\end{minipage}
\hfill
\begin{minipage}[t]{0.46\textwidth}
\centering
\vspace{0.5em}
\includegraphics[width=\textwidth]{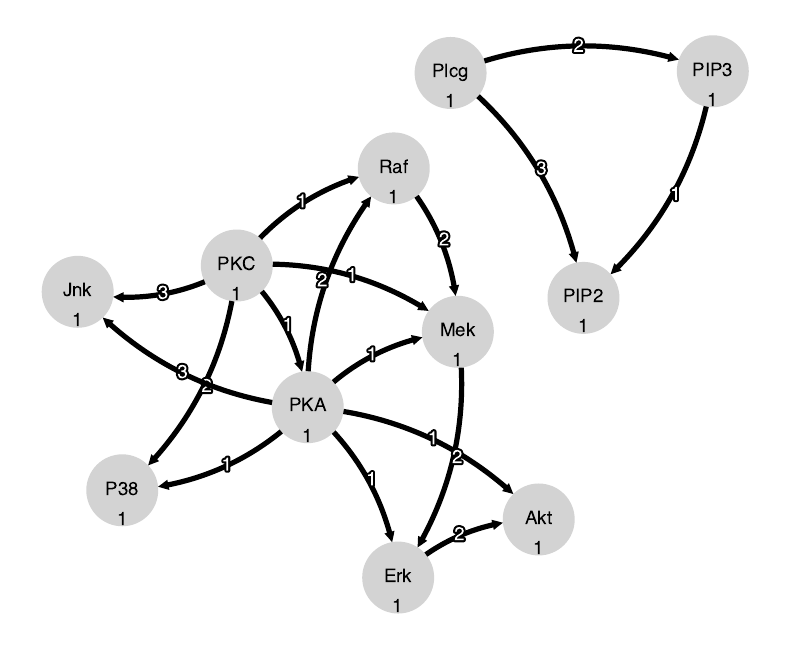}
\end{minipage}

\vspace{1em}

\paragraph{6-Node Linear VAR.}
Hub variable $X_0$, $L{=}3$, $c = \sqrt{2}$.
$X_1$ and $X_2$ have \emph{no lag-1 self-dependency}, violating \cedar{}'s key assumption and testing robustness.

\noindent
\begin{minipage}[t]{0.55\textwidth}
\small
\begin{align*}
X_0^t &= 0.5c\,X_0^{t-1} + \varepsilon_0 \\
X_1^t &= 0.5\,X_0^{t-2} + \varepsilon_1 \\
X_2^t &= -0.4\,X_0^{t-3} + \varepsilon_2 \\
X_3^t &= 0.25c\,X_3^{t-1} - 0.5\,X_0^{t-2} + 0.45c\,X_4^{t-1} + \varepsilon_3 \\
X_4^t &= 0.25c\,X_4^{t-1} - 0.5c\,X_3^{t-1} + \varepsilon_4 \\
X_5^t &= 0.65\,X_5^{t-1} - 0.1\,X_2^{t-1} - 0.55\,X_4^{t-2} + \varepsilon_5
\end{align*}
\end{minipage}
\hfill
\begin{minipage}[t]{0.42\textwidth}
\centering
\vspace{1em}
\includegraphics[height=3.5cm]{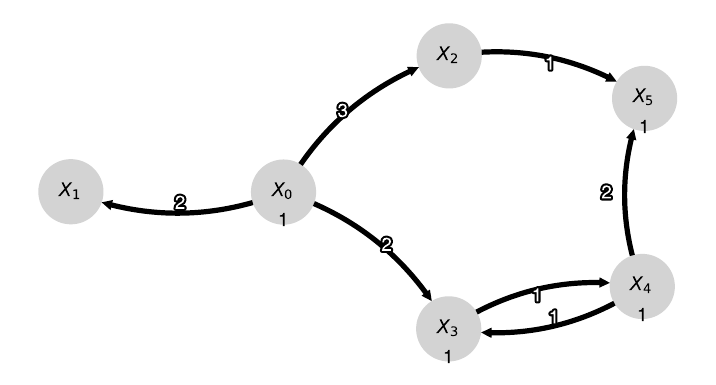}
\end{minipage}

\vspace{1em}

\paragraph{5-Node Nonlinear System.}
Quadratic nonlinearities, $L{=}3$, $c = \sqrt{2}$.
$X_0$ has AR(2) structure, violating the single-lag assumption. Uses SplitKCI for all methods.

\noindent
\begin{minipage}[t]{0.55\textwidth}
\small
\begin{align*}
X_0^t &= 0.95c\,X_0^{t-1} - 0.9025\,X_0^{t-2} + \varepsilon_0 \\
X_1^t &= 0.5\,(X_0^{t-2})^2 + \varepsilon_1 \\
X_2^t &= -0.4\,X_0^{t-3} + \varepsilon_2 \\
X_3^t &= 0.25c\,X_3^{t-1} - 0.5\,(X_0^{t-2})^2 + 0.25c\,X_4^{t-1} + \varepsilon_3 \\
X_4^t &= 0.25c\,X_4^{t-1} - 0.25c\,X_3^{t-1} + \varepsilon_4
\end{align*}
\end{minipage}
\hfill
\begin{minipage}[t]{0.42\textwidth}
\centering
\vspace{1em}
\includegraphics[height=4cm]{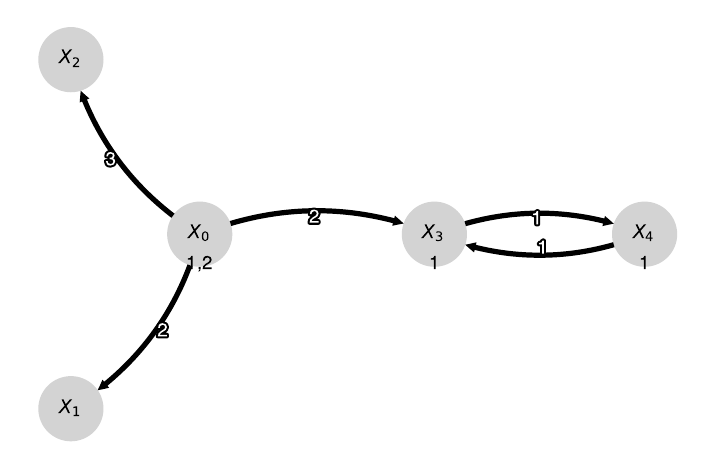}
\end{minipage}

\vspace{1em}

\paragraph{Confounded Trends, 7-Node.}
Shared quadratic trend $\tau^t = 10 \cdot (t/T)^2$, $\varepsilon_i \sim \mathcal{N}(0, 0.09)$.
Only true cross-variable edge: $X_5 \to X_6$ at lag~2; all others confounded by $\tau$.
\cedar{}, \cdnots{}, \cdnots{}+ use \texttt{include\_C=True, c\_preset=linear+quad}; PCMCI+ lacks C-node support.

\noindent
\begin{minipage}[t]{0.38\textwidth}
\small
\begin{align*}
X_0^t &= \tau^t + \varepsilon_0 \\
X_1^t &= 0.9\,\tau^t + \varepsilon_1 \\
X_2^t &= 0.5\,X_2^{t-1} + 0.85\,\tau^t + \varepsilon_2 \\
X_3^t &= 0.8\,\tau^t + \varepsilon_3 \\
X_4^t &= 0.6\,X_4^{t-1} + 0.75\,\tau^t + \varepsilon_4 \\
X_5^t &= 0.5\,X_5^{t-1} + \varepsilon_5 \\
X_6^t &= 0.4\,X_6^{t-1} + 0.6\,X_5^{t-2} + \varepsilon_6
\end{align*}
\end{minipage}
\hfill
\begin{minipage}[t]{0.58\textwidth}
\centering
\vspace{0.5em}
\includegraphics[height=4.5cm]{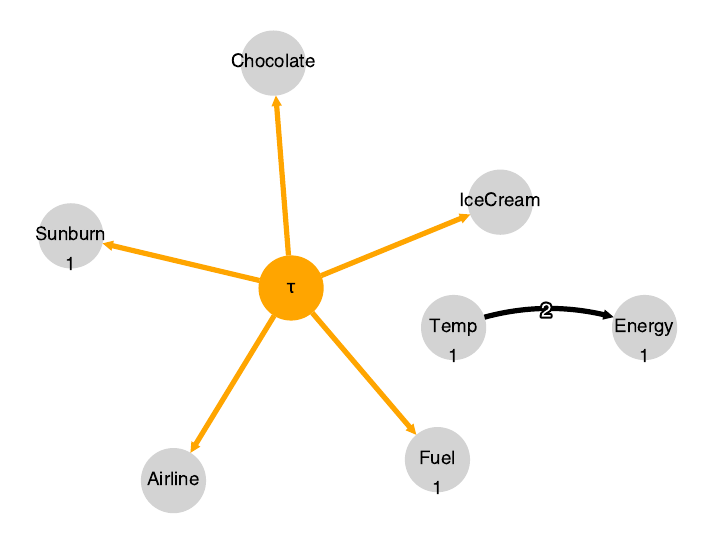}
\end{minipage}

\vspace{0.5em}

\cedar{} matches all baselines on the cascade (F1$=$0.944) and achieves perfect F1 on confounded trends---PCMCI+ degrades to F1$=$0.35 without trend handling.
On the linear benchmark, \cedar{} trades recall for precision (1.00 precision, 0.69 recall).
Figure~\ref{fig:compare} shows \cedar{}'s discovered graphs on all four benchmarks.

\begin{figure}[h]
\centering
\begin{subfigure}[t]{0.48\textwidth}
\centering
\includegraphics[height=5.5cm]{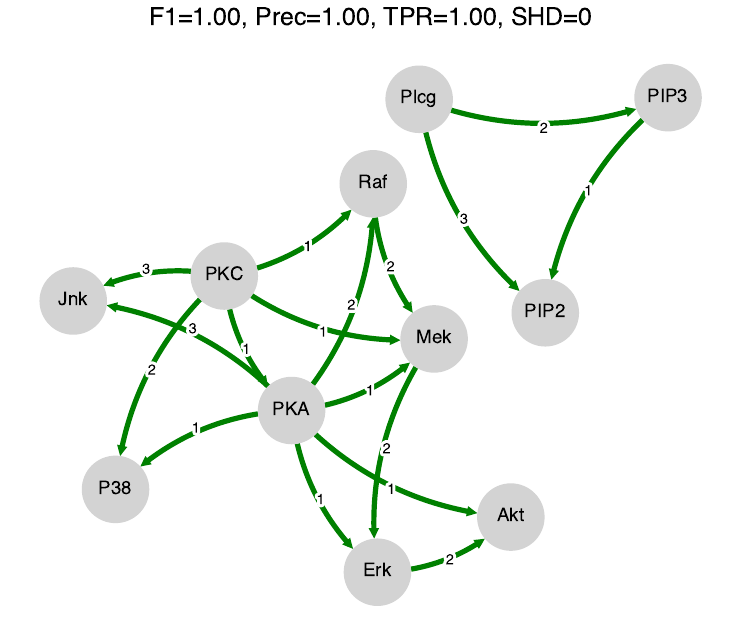}
\caption{Cascade Network (11-Node)}
\end{subfigure}
\hfill
\begin{subfigure}[t]{0.48\textwidth}
\centering
\includegraphics[height=5.5cm]{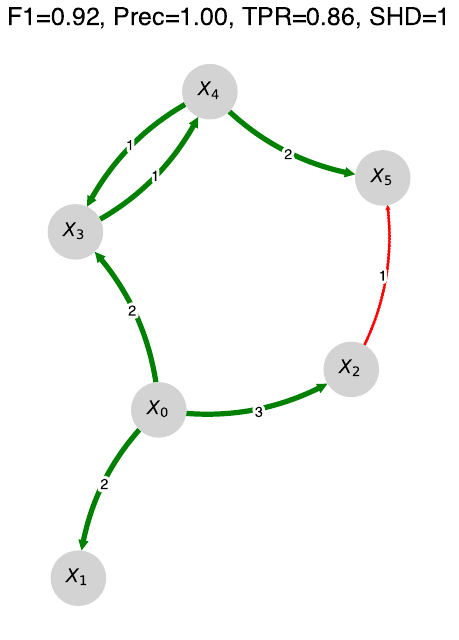}
\caption{6-Node Linear VAR}
\end{subfigure}

\vspace{0.5em}

\begin{subfigure}[t]{0.48\textwidth}
\centering
\includegraphics[height=5.5cm]{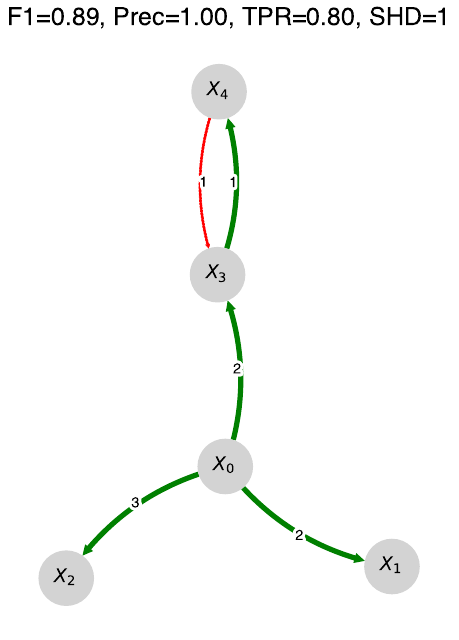}
\caption{5-Node Nonlinear System}
\end{subfigure}
\hfill
\begin{subfigure}[t]{0.48\textwidth}
\centering
\includegraphics[height=5.5cm]{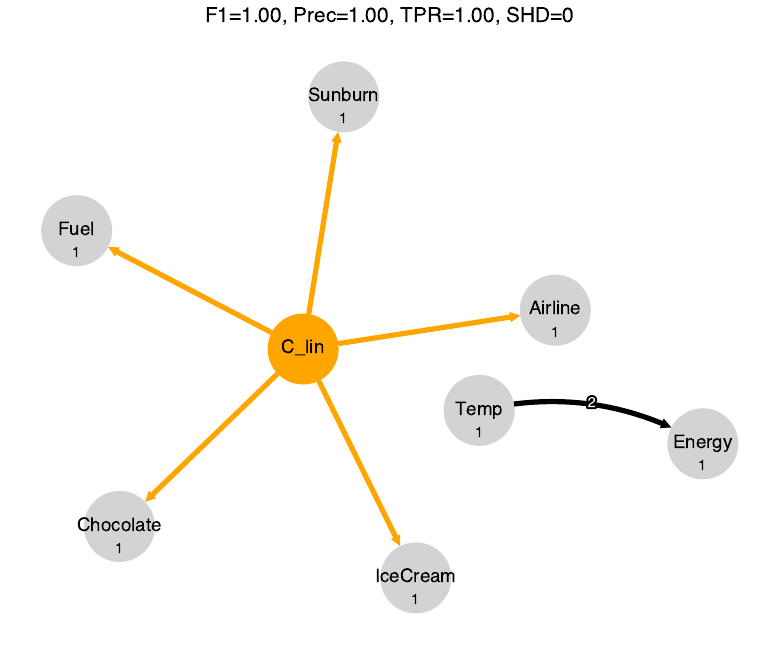}
\caption{Confounded Trends (7-Node, \texttt{c\_preset=linear} for visual clarity; Table~\ref{tab:main} uses \texttt{linear+quad})}
\end{subfigure}
\caption{\cedar{}'s discovered graph vs.\ ground truth (seed 42). Green: TP, red dashed: FN, orange: FP.}
\label{fig:compare}
\end{figure}

\begin{table}[h]
\centering
\caption{Named benchmarks (mean$\pm$s.e.\ over 5 seeds). All methods use partial correlation at $\alpha=0.01$ except the 5-node nonlinear benchmark, which uses SplitKCI. Metrics report cross-variable edges only (self-loops excluded); \cedar{} assumes AR(1) self-dynamics by design and adds self-loops for all variables without validation, so including them would count as false positives for variables without true self-lags. Confounded-trends: \cedar{}, \cdnots{}, and \cdnots{}+ use \texttt{include\_C=True}, \texttt{c\_preset=linear+quad}; PCMCI+ has no C-node support.}
\label{tab:main}
\small
\begin{tabular}{l ccc ccc}
\toprule
& \multicolumn{3}{c}{\textbf{Cascade Network} (11-node, $T{=}300$)} & \multicolumn{3}{c}{\textbf{6-Node Linear VAR} ($T{=}300$)} \\
\cmidrule(lr){2-4} \cmidrule(lr){5-7}
\textbf{Method} & F1 & Prec & Rec & F1 & Prec & Rec \\
\midrule
\textbf{\cedar{}} & \textbf{.944}{\tiny$\pm$.03} & .975{\tiny$\pm$.03} & \textbf{.918}{\tiny$\pm$.04} & .809{\tiny$\pm$.04} & \textbf{1.00}{\tiny$\pm$.00} & .686{\tiny$\pm$.05} \\
\cdnots{}         & .826{\tiny$\pm$.01} & .920{\tiny$\pm$.03} & .753{\tiny$\pm$.01} & \textbf{.923}{\tiny$\pm$.00} & \textbf{1.00}{\tiny$\pm$.00} & \textbf{.857}{\tiny$\pm$.00} \\
\cdnots{}+        & .933{\tiny$\pm$.02} & .963{\tiny$\pm$.03} & .906{\tiny$\pm$.02} & .910{\tiny$\pm$.01} & .971{\tiny$\pm$.03} & \textbf{.857}{\tiny$\pm$.00} \\
PCMCI+            & \textbf{.944}{\tiny$\pm$.02} & \textbf{.987}{\tiny$\pm$.01} & .906{\tiny$\pm$.02} & .897{\tiny$\pm$.02} & .943{\tiny$\pm$.04} & \textbf{.857}{\tiny$\pm$.00} \\
\bottomrule
\end{tabular}

\vspace{0.8em}

\begin{tabular}{l ccc ccc}
\toprule
& \multicolumn{3}{c}{\textbf{5-Node Nonlinear} ($T{=}300$, SplitKCI)} & \multicolumn{3}{c}{\textbf{Confounded Trends} (7-node, $T{=}500$, with C)} \\
\cmidrule(lr){2-4} \cmidrule(lr){5-7}
\textbf{Method} & F1 & Prec & Rec & F1 & Prec & Rec \\
\midrule
\textbf{\cedar{}} & .844{\tiny$\pm$.04} & .950{\tiny$\pm$.05} & \textbf{.760}{\tiny$\pm$.04} & \textbf{1.00}{\tiny$\pm$.00} & \textbf{1.00}{\tiny$\pm$.00} & \textbf{1.00}{\tiny$\pm$.00} \\
\cdnots{}         & .662{\tiny$\pm$.04} & .950{\tiny$\pm$.05} & .520{\tiny$\pm$.05} & .900{\tiny$\pm$.10} & .867{\tiny$\pm$.13} & \textbf{1.00}{\tiny$\pm$.00} \\
\cdnots{}+        & \textbf{.861}{\tiny$\pm$.03} & \textbf{1.00}{\tiny$\pm$.00} & \textbf{.760}{\tiny$\pm$.04} & .800{\tiny$\pm$.08} & .700{\tiny$\pm$.12} & \textbf{1.00}{\tiny$\pm$.00} \\
PCMCI+            & .833{\tiny$\pm$.04} & \textbf{1.00}{\tiny$\pm$.00} & .720{\tiny$\pm$.05} & .348{\tiny$\pm$.05} & .215{\tiny$\pm$.04} & \textbf{1.00}{\tiny$\pm$.00} \\
\bottomrule
\end{tabular}
\end{table}

\section{Full Scaling Results}
\label{sec:full_results}

Figure~\ref{fig:shd_vs_d} reports SHD across dimensionalities (lower is better); Table~\ref{tab:full_scaling} reports F1 scores (mean and standard deviation over 20 seeds) for all sample sizes tested.

\begin{figure}[h]
\centering
\includegraphics[width=\textwidth]{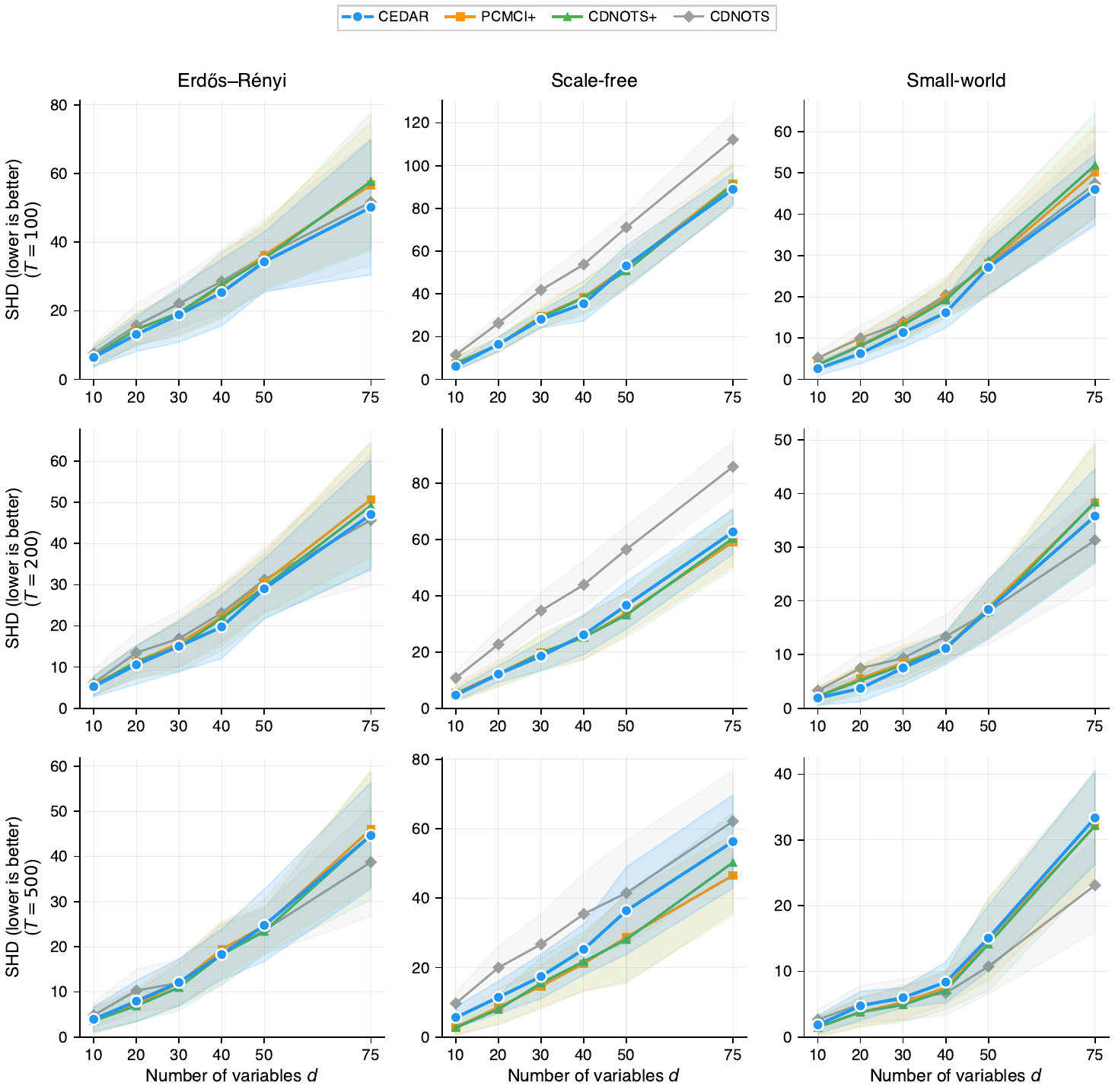}
\caption{Structural Hamming distance (lower is better) vs.\ dimensionality $d$ for three graph topologies and three sample sizes. Same layout as Figure~\ref{fig:f1_vs_d}.}
\label{fig:shd_vs_d}
\end{figure}

\begin{table}[h]
\centering
\caption{Runtime in seconds ($T=300$, mean over 20 seeds $\times$ 3 topologies, ParCorr on NVIDIA L4 GPU). All methods use the same CI test. The difference reflects conditioning-set size, not per-test cost.}
\label{tab:runtime}
\small
\begin{tabular}{l rrrrrr}
\toprule
\textbf{Method} & $d=10$ & $d=20$ & $d=30$ & $d=40$ & $d=50$ & $d=75$ \\
\midrule
\textbf{\cedar{}}  & \textbf{0.21} & \textbf{0.73} & \textbf{1.48} & \textbf{2.72} & \textbf{3.86} & \textbf{8.04} \\
\cdnots{}          & 0.21 & 0.84 & 1.69 & 3.68 & 4.32 & 10.57 \\
\cdnots{}+         & 0.33 & 1.18 & 2.39 & 5.26 & 5.98 & 15.86 \\
PCMCI+             & 0.42 & 1.48 & 2.99 & 6.53 & 7.48 & 19.37 \\
\bottomrule
\end{tabular}
\end{table}

\begin{sidewaystable}
\centering
\caption{F1 on random SCPs (mean(sd) over 20 seeds) for all sample sizes. Best per column in \textbf{bold}. Topologies: Erd\H{o}s--R\'enyi (sparse), scale-free (BA, $m{=}2$), small-world (WS, $k{=}3$).}
\label{tab:full_scaling}
\scriptsize
\setlength{\tabcolsep}{3pt}
\begin{tabular}{l cccccc cccccc cccccc}
\toprule
& \multicolumn{6}{c}{\textbf{Erd\H{o}s--R\'enyi}} & \multicolumn{6}{c}{\textbf{Scale-free}} & \multicolumn{6}{c}{\textbf{Small-world}} \\
\cmidrule(lr){2-7} \cmidrule(lr){8-13} \cmidrule(lr){14-19}
& $d{=}10$ & $d{=}20$ & $d{=}30$ & $d{=}40$ & $d{=}50$ & $d{=}75$ & $d{=}10$ & $d{=}20$ & $d{=}30$ & $d{=}40$ & $d{=}50$ & $d{=}75$ & $d{=}10$ & $d{=}20$ & $d{=}30$ & $d{=}40$ & $d{=}50$ & $d{=}75$ \\
\midrule
\multicolumn{19}{l}{\textit{$T = 100$}} \\
\textbf{\cedar{}} & \textbf{.63}(.15) & \textbf{.65}(.09) & \textbf{.67}(.08) & \textbf{.64}(.08) & \textbf{.59}(.07) & \textbf{.56}(.09) & \textbf{.77}(.09) & .73(.07) & \textbf{.70}(.05) & \textbf{.73}(.07) & .67(.07) & .63(.04) & \textbf{.85}(.11) & \textbf{.82}(.08) & \textbf{.79}(.08) & \textbf{.78}(.06) & \textbf{.71}(.08) & \textbf{.68}(.06) \\
\cdnots{}+ & .62(.19) & .62(.07) & .66(.05) & .62(.08) & .58(.06) & .53(.08) & .72(.12) & \textbf{.73}(.06) & .69(.06) & .71(.07) & \textbf{.69}(.06) & \textbf{.64}(.04) & .79(.09) & .77(.08) & .76(.08) & .74(.06) & \textbf{.70}(.09) & .66(.07) \\
PCMCI+ & \textbf{.63}(.19) & .61(.07) & .66(.05) & .62(.07) & .57(.07) & .53(.08) & .71(.12) & \textbf{.74}(.07) & .69(.06) & .70(.07) & .68(.06) & \textbf{.63}(.04) & .79(.10) & .77(.08) & .75(.07) & .74(.06) & \textbf{.71}(.08) & .67(.06) \\
\cdnots{} & .55(.16) & .57(.11) & .58(.09) & .57(.09) & .56(.07) & .53(.08) & .51(.08) & .50(.08) & .50(.07) & .52(.08) & .50(.05) & .49(.06) & .67(.12) & .70(.10) & .73(.09) & .71(.06) & .70(.07) & .67(.07) \\
\midrule
\multicolumn{19}{l}{\textit{$T = 200$}} \\
\textbf{\cedar{}} & \textbf{.72}(.11) & \textbf{.73}(.10) & .75(.06) & \textbf{.74}(.08) & .67(.07) & \textbf{.62}(.06) & \textbf{.84}(.08) & .82(.05) & \textbf{.83}(.05) & .82(.05) & .80(.05) & .78(.03) & \textbf{.89}(.08) & \textbf{.90}(.06) & \textbf{.87}(.06) & \textbf{.86}(.04) & .82(.05) & .78(.05) \\
\cdnots{}+ & .71(.13) & .72(.08) & \textbf{.76}(.05) & .71(.08) & \textbf{.67}(.07) & .61(.07) & .83(.09) & \textbf{.83}(.06) & .82(.06) & \textbf{.83}(.06) & \textbf{.82}(.05) & \textbf{.79}(.04) & .87(.11) & .86(.08) & .86(.05) & \textbf{.86}(.04) & \textbf{.83}(.04) & .77(.06) \\
PCMCI+ & .70(.15) & .72(.08) & .74(.05) & .71(.07) & .66(.06) & .60(.07) & .82(.11) & .82(.06) & .82(.05) & \textbf{.83}(.06) & \textbf{.82}(.04) & \textbf{.79}(.03) & .87(.11) & .85(.07) & .86(.06) & \textbf{.86}(.04) & .82(.05) & .77(.06) \\
\cdnots{} & .67(.17) & .65(.08) & .70(.07) & .67(.08) & .64(.06) & .61(.07) & .57(.10) & .60(.11) & .62(.07) & .65(.08) & .65(.06) & .66(.03) & .80(.11) & .79(.08) & .83(.06) & .82(.06) & .81(.06) & \textbf{.80}(.05) \\
\midrule
\multicolumn{19}{l}{\textit{$T = 300$}} \\
\textbf{\cedar{}} & \textbf{.78}(.12) & .78(.09) & \textbf{.79}(.07) & \textbf{.74}(.07) & \textbf{.70}(.07) & .63(.07) & .88(.09) & .86(.05) & .83(.04) & .84(.05) & .82(.05) & .81(.03) & .91(.08) & \textbf{.92}(.06) & \textbf{.90}(.05) & .88(.04) & .85(.04) & .81(.05) \\
\cdnots{}+ & .77(.12) & \textbf{.79}(.07) & .79(.05) & \textbf{.74}(.08) & .69(.07) & .63(.07) & \textbf{.90}(.08) & \textbf{.86}(.06) & \textbf{.85}(.06) & \textbf{.86}(.05) & \textbf{.85}(.04) & .82(.03) & \textbf{.92}(.06) & .89(.07) & .90(.04) & \textbf{.89}(.05) & .85(.04) & .81(.04) \\
PCMCI+ & .76(.13) & \textbf{.79}(.07) & .78(.05) & \textbf{.74}(.07) & .68(.07) & .62(.07) & \textbf{.89}(.08) & \textbf{.86}(.07) & \textbf{.84}(.06) & \textbf{.86}(.05) & \textbf{.85}(.04) & \textbf{.83}(.03) & \textbf{.92}(.06) & .88(.07) & .89(.05) & \textbf{.89}(.05) & .85(.04) & .81(.04) \\
\cdnots{} & .70(.15) & .71(.09) & .73(.07) & .71(.09) & .68(.07) & \textbf{.64}(.08) & .61(.11) & .64(.08) & .66(.09) & .69(.10) & .71(.07) & .72(.04) & .82(.10) & .82(.08) & .87(.06) & .86(.06) & \textbf{.86}(.03) & \textbf{.83}(.04) \\
\midrule
\multicolumn{19}{l}{\textit{$T = 500$}} \\
\textbf{\cedar{}} & .81(.10) & .82(.08) & .81(.06) & .77(.06) & .74(.06) & .65(.07) & .81(.10) & .84(.07) & .85(.05) & .84(.04) & .82(.06) & .82(.04) & .90(.07) & .88(.05) & .90(.03) & .90(.03) & .86(.04) & .81(.04) \\
\cdnots{}+ & \textbf{.84}(.12) & \textbf{.84}(.07) & \textbf{.83}(.06) & \textbf{.77}(.07) & \textbf{.75}(.06) & .67(.06) & \textbf{.91}(.07) & \textbf{.89}(.06) & .87(.05) & \textbf{.87}(.05) & \textbf{.86}(.05) & .85(.04) & \textbf{.92}(.08) & \textbf{.90}(.06) & \textbf{.92}(.04) & \textbf{.92}(.03) & .87(.06) & .82(.04) \\
PCMCI+ & .83(.13) & .84(.08) & .81(.05) & .76(.07) & .74(.06) & .66(.06) & \textbf{.91}(.07) & .89(.06) & \textbf{.88}(.05) & \textbf{.87}(.05) & \textbf{.86}(.05) & \textbf{.85}(.03) & \textbf{.92}(.08) & \textbf{.90}(.06) & .92(.04) & \textbf{.91}(.03) & .87(.05) & .82(.04) \\
\cdnots{} & .75(.14) & .75(.09) & .80(.05) & .76(.08) & .73(.06) & \textbf{.68}(.08) & .64(.12) & .68(.10) & .74(.09) & .75(.09) & .77(.09) & .78(.05) & .85(.10) & .87(.08) & .90(.06) & \textbf{.92}(.04) & \textbf{.90}(.04) & \textbf{.86}(.04) \\
\midrule
\multicolumn{19}{l}{\textit{$T = 1000$}} \\
\textbf{\cedar{}} & .85(.09) & .85(.06) & .84(.07) & .79(.07) & .76(.07) & .68(.09) & .78(.10) & .79(.07) & .79(.09) & .80(.05) & .79(.08) & .79(.04) & .92(.07) & .89(.07) & .89(.04) & .88(.04) & .84(.04) & .79(.04) \\
\cdnots{}+ & \textbf{.89}(.09) & \textbf{.87}(.06) & \textbf{.87}(.05) & \textbf{.82}(.07) & \textbf{.80}(.07) & .71(.07) & \textbf{.93}(.06) & \textbf{.88}(.05) & \textbf{.87}(.06) & .86(.04) & \textbf{.86}(.06) & .82(.04) & \textbf{.94}(.06) & \textbf{.92}(.05) & \textbf{.94}(.03) & .92(.03) & .88(.03) & .82(.03) \\
PCMCI+ & .88(.10) & .86(.07) & \textbf{.87}(.06) & .81(.07) & .80(.07) & .70(.08) & .92(.06) & .87(.05) & \textbf{.87}(.07) & \textbf{.86}(.04) & \textbf{.86}(.05) & \textbf{.83}(.03) & \textbf{.94}(.06) & \textbf{.92}(.06) & .93(.03) & .91(.03) & .88(.03) & .82(.02) \\
\cdnots{} & .79(.13) & .82(.08) & .84(.04) & \textbf{.82}(.06) & .79(.06) & \textbf{.73}(.07) & .70(.10) & .75(.09) & .78(.10) & .80(.08) & .81(.08) & .81(.05) & .88(.09) & .90(.07) & .92(.05) & \textbf{.92}(.04) & \textbf{.91}(.03) & \textbf{.87}(.04) \\
\bottomrule
\end{tabular}
\end{sidewaystable}

\section{Ablation Across Graph Topologies}
\label{sec:ablation_topologies}

Table~\ref{tab:ablation_topologies} extends the ablation of Table~\ref{tab:ablation} to three random graph topologies: Erd\H{o}s--R\'{e}nyi (ER), scale-free (SF), and small-world (SW), at $d{=}20$, $T{=}200$, mean over 10 seeds.
The non-monotone pattern and the load-bearing role of MCI pruning and AR(1)-residualized dcor hold consistently across all three topologies.
The baseline uses top-1 lag selection without significance thresholding, mimicking \sypi{}'s Lasso-based approach.

\begin{table}[h]
\centering
\caption{Incremental ablation across three graph topologies ($d{=}20$, $T{=}200$, $\alpha{=}0.01$, mean F1 over 10 seeds). MCI pruning and AR(1)-residualized dcor are load-bearing across all topologies.}
\label{tab:ablation_topologies}
\small
\begin{tabular}{l ccc}
\toprule
\textbf{Configuration} & ER & Scale-free & Small-world \\
\midrule
Baseline (\sypi{}-style)          & .655 & .524 & .658 \\
\quad + $\mathcal{U}$-centered dcor & .672 & .521 & .675 \\
\quad + Analytic $t$-test          & .682 & .531 & .674 \\
\quad + Multi-lag                  & .681 & .481 & .665 \\
\quad + MCI pruning                & .775 & .761 & .777 \\
\quad + AR(1)-res.\ dcor (\textbf{\cedar{}}) & \textbf{.773} & \textbf{.855} & \textbf{.784} \\
\bottomrule
\end{tabular}
\end{table}

\section{Robustness Under Assumption Violations}
\label{sec:stress_tests}

Table~\ref{tab:stress} reports \cedar{} and \cdnots{}+ performance on three stress tests that deliberately violate \cedar{}'s key assumptions ($T{=}300$, mean over 10 seeds).

\begin{table}[h]
\centering
\caption{Stress tests under assumption violations ($d{=}15$--$20$, $T{=}300$, mean over 10 seeds). The gap between \cedar{} and \cdnots{}+ grows as violations accumulate, but \cedar{} degrades predictably rather than catastrophically.}
\label{tab:stress}
\small
\begin{tabular}{l l ccccc}
\toprule
\textbf{Violation} & \textbf{Method} & TP & FP & FN & Prec & F1 \\
\midrule
\multirow{2}{*}{A7 bypass $\times$5 ($d{=}15$, 15 edges)} & \cedar{} & 14.4 & 1.8 & 0.6 & .893 & .924 \\
 & \cdnots{}+ & 15.0 & 1.1 & 0.0 & .934 & .965 \\
\midrule
\multirow{2}{*}{Dense AR(2) ($d{=}15$, 18 edges, A3)} & \cedar{} & 16.2 & 2.4 & 1.8 & .875 & .886 \\
 & \cdnots{}+ & 17.9 & 1.3 & 0.1 & .934 & .963 \\
\midrule
\multirow{2}{*}{Mixed A3+A5+A7 ($d{=}20$, 27 edges)} & \cedar{} & 22.4 & 2.4 & 4.6 & .905 & .865 \\
 & \cdnots{}+ & 27.0 & 1.2 & 0.0 & .958 & .979 \\
\bottomrule
\end{tabular}
\end{table}

The gap between \cedar{} and \cdnots{}+ grows as violations accumulate: from 0.041 F1 points (single violation type, A7 bypasses) to 0.114 F1 points (all three violations combined in $d{=}20$).
\cedar{} degrades predictably rather than catastrophically---F1 ranges from 0.865 to 0.924 even under simultaneous violations---reflecting the bounded nature of its structural assumptions.
\cdnots{}+, with its iterative PCMCI+-style conditioning, is more robust to these violations at the cost of requiring larger $T$ and more CI tests.

\section{\sypi{}-All-Targets Comparison}
\label{sec:sypi_comparison}

Table~\ref{tab:sypi_comparison} compares \cedar{} against a faithful \sypi{}-all-targets adaptation---the closest ancestor method---to directly establish the contribution of each component.
\sypi{}-all-targets runs \cedar{}'s two-condition CI test for every variable as target, using LassoCV (cv=5) for lag selection with the selected lag taken as the argmax of absolute Lasso coefficients (matching \citet{mastakouri2021sypi}, Section~6.5), $\alpha_1{=}\alpha_2{=}0.05$ (from the paper's threshold sweep), and no MCI pruning.
CEDAR-no-MCI uses all \cedar{} components except MCI pruning, establishing the independent contribution of pruning.

\begin{table}[h]
\centering
\caption{\sypi{}-all-targets vs.\ \cedar{} (mean F1 over 10 seeds, $\max\_\text{lag}{=}3$, ParCorr).
\sypi{}-all-targets uses LassoCV lag selection, $\alpha_1{=}\alpha_2{=}0.05$, no MCI pruning.
CEDAR-no-MCI uses AR(1)-residualized dcor, $\alpha{=}0.01$, no MCI pruning.
Full \cedar{} adds stable MCI pruning.
Mean precision of \sypi{}-all-targets is 0.34; of Full \cedar{} is 0.87.}
\label{tab:sypi_comparison}
\small
\begin{tabular}{l ccc}
\toprule
\textbf{Benchmark} & \sypi{}-all-targets & CEDAR-no-MCI & \textbf{\cedar{}} \\
\midrule
ER $d{=}20$, $T{=}100$  & .439 & .612 & \textbf{.625} \\
ER $d{=}20$, $T{=}200$  & .498 & .694 & \textbf{.773} \\
ER $d{=}20$, $T{=}500$  & .502 & .705 & \textbf{.856} \\
\midrule
SF $d{=}20$, $T{=}100$  & .432 & .648 & \textbf{.781} \\
SF $d{=}20$, $T{=}200$  & .436 & .594 & \textbf{.855} \\
SF $d{=}20$, $T{=}500$  & .397 & .503 & \textbf{.847} \\
\midrule
SF $d{=}40$, $T{=}200$  & .373 & .536 & \textbf{.830} \\
\midrule
Cascade ($d{=}11$), $T{=}200$ & .609 & .851 & \textbf{.887} \\
\bottomrule
\end{tabular}
\end{table}

\paragraph{Interpretation.}
\sypi{}-all-targets achieves consistently high recall (0.64--0.86) because Lasso finds candidate lags, but suffers low precision (0.24--0.53) because it does not residualize the AR(1) component before scoring lags, causing the target's autocorrelation to inflate cross-variable importance scores.
The gap between CEDAR-no-MCI and \sypi{}-all-targets reflects the contribution of AR(1)-residualized dcor: even without MCI pruning, better lag screening yields substantially higher F1.
The further gap between CEDAR-no-MCI and Full \cedar{} reflects MCI pruning, which removes indirect edges that survive the two-condition test---this is the dominant gain at larger $d$ and $T$.
Crucially, the F1 advantage of \cedar{} over \sypi{}-all-targets grows with $d$ (0.186 at $d{=}20$ ER vs.\ 0.457 at $d{=}40$ SF), confirming that both components scale favorably.

\section{Elbe River Benchmark Details}
\label{sec:elbe_details}

\paragraph{Dataset and challenges.}
The Elbe River benchmark from CausalRivers~\citep{stein2025causalrivers} consists of 12 gauging stations along the Elbe main branch, resampled to 3-hour resolution ($T{=}11{,}688$, $d{=}12$, 5.3\% missing values handled by linear interpolation).
The dataset is challenging for three reasons:
(i)~\textbf{Nonstationarity}: hydrological regimes change the causal lag structure---flood-season propagation is fast (lag 1--2) while dry-season propagation is slow (lag 5+).
(ii)~\textbf{Transient confounding}: rainfall simultaneously raises discharge at multiple stations, creating spurious correlations that violate causal sufficiency; this effect is weakest during low-flow (i.e., dry season) conditions where baseflow propagation dominates.
(iii)~\textbf{Variable propagation speed}: the causal lag varies across station pairs (distance-dependent) and across regimes (flowrate-dependent).

\paragraph{Setup.}
We apply PELT changepoint detection on the mean discharge signal with minimum segment size 300 samples and segment threshold 0.25.
\cedar{} is run independently on each of the three regimes with lag bounds $L{=}8,5,3$ (low/normal/high-flow), AR(1)-relative lag filtering ($\alpha_{\text{lag}}{=}0.97$), $\alpha{=}0.01$, and ParCorr as the CI test; per-regime edges are combined by OR aggregation.
The regime segmentation is visualized in Figure~\ref{fig:elbe_regimes}.
\cdnots{}+ is run with the same regime-conditional setup for comparison.

\begin{figure}[h]
\centering
\includegraphics[width=\textwidth]{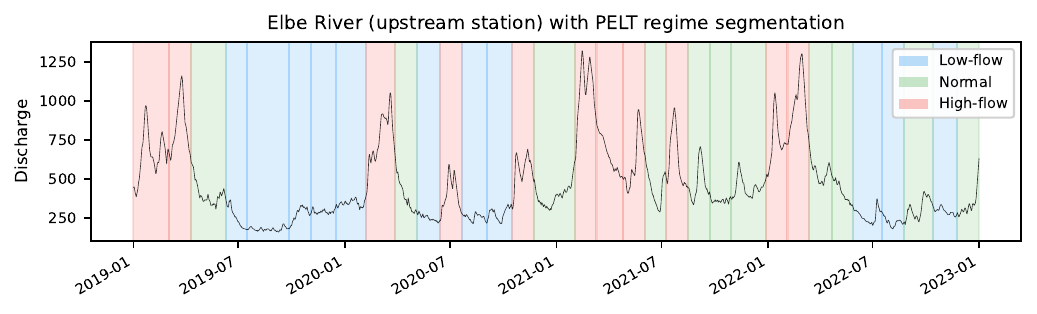}
\caption{Elbe River discharge (upstream station, S166) with PELT-detected regime segmentation into three hydrological regimes. Low-flow periods provide the cleanest causal signal due to reduced transient rainfall confounding.}
\label{fig:elbe_regimes}
\end{figure}

\begin{table}[h]
\centering
\caption{Elbe River benchmark ($d{=}12$, $T{=}11{,}688$, 11 true edges). Regime-conditional discovery substantially outperforms single-run baselines. Low-flow yields the highest precision for \cedar{}, consistent with reduced transient confounding.}
\label{tab:elbe}
\small
\begin{tabular}{l ccccc c}
\toprule
\textbf{Method} & TP & FP & FN & Prec & Rec & F1 \\
\midrule
\cedar{} (single, $L{=}5$) & 0 & 14 & 11 & .000 & .000 & .000 \\
\cdnots{}+ (single, $L{=}5$) & 5 & 8 & 6 & .385 & .455 & .417 \\
\midrule
\cedar{} + PELT (agg.) & \textbf{10} & 3 & \textbf{1} & .769 & \textbf{.909} & \textbf{.833} \\
\quad low-flow only & 9 & \textbf{1} & 2 & \textbf{.900} & .818 & .857 \\
\quad normal only & 5 & 2 & 6 & .714 & .455 & .556 \\
\quad high-flow only & 4 & 3 & 7 & .571 & .364 & .444 \\
\midrule
\cdnots{}+ + PELT (agg.) & 9 & 4 & 2 & .692 & .818 & .750 \\
\quad low-flow only & 8 & 3 & 3 & .727 & .727 & .727 \\
\bottomrule
\end{tabular}
\end{table}

The per-regime breakdown confirms the transient confounder hypothesis: the low-flow regime achieves 0.900 precision (9 TP, 1 FP) compared to 0.571 for high-flow.
The single missed edge in the aggregated \cedar{} result is between two distant stations where the propagation lag exceeds all three regime-specific bounds.

\end{document}